\documentclass[11pt]{article}
\usepackage{graphicx,amsmath,amsfonts,amssymb,bm,hyperref,url,breakurl,epsfig,epsf,color,fullpage,MnSymbol,mathbbol,fmtcount,algorithmic,algorithm,semtrans,cite,caption,subcaption,multirow}

\usepackage[title]{appendix}

\usepackage{graphicx}
\usepackage{caption}
\usepackage{subcaption}
\usepackage[bottom,hang,flushmargin]{footmisc}
\usepackage{comment}
\usepackage[bottom]{footmisc}
\usepackage{caption}
\usepackage{enumitem}

\usepackage{graphicx}
\usepackage{hyperref}
\usepackage{url}            
\usepackage{mathtools}
\usepackage{booktabs}
\usepackage{amsfonts}       
\usepackage{nicefrac}       
\usepackage{caption}
\usepackage{hyperref}       
\usepackage{tabularx,tabulary}
\usepackage{graphicx}
\usepackage{amsthm}
\usepackage{amsmath}
\usepackage{amssymb}
\usepackage{enumitem}
\usepackage{algorithm}
\usepackage{algorithmic}

\newtheorem{mydef}{Definition}
\newtheorem{thm}{Theorem}
\newtheorem{remark}{Remark}
\newtheorem*{remark*}{Remark}
\newtheorem{lemma}{Lemma}

\newtheorem{prop}{Proposition}

\newcommand*\diff{\mathop{}\!\mathrm{d}}

\usepackage{hyperref}
\definecolor{darkred}{RGB}{150,0,0}
\definecolor{darkgreen}{RGB}{0,150,0}
\definecolor{darkblue}{RGB}{0,0,200}
\hypersetup{colorlinks=true, linkcolor=darkred, citecolor=darkgreen, urlcolor=darkblue}

\numberwithin{equation}{section}

\title{{\huge GANs May Have No Nash Equilibria}}

\author{\\ Farzan Farnia, Asuman Ozdaglar\\
\{farnia,asuman\}@mit.edu\\\\
Massachusetts Institute of Technology\\}
\date{}

\begin{document}
\maketitle




\begin{abstract}
Generative adversarial networks (GANs) represent a zero-sum game between two machine players, a generator and a discriminator, designed to learn the distribution of data. While GANs have achieved state-of-the-art performance in several benchmark learning tasks, GAN minimax optimization still poses great theoretical and empirical challenges. GANs trained using first-order optimization methods commonly fail to converge to a stable solution where the players cannot improve their objective, i.e., the Nash equilibrium of the underlying game. Such issues raise the question of the existence of Nash equilibrium solutions in the GAN zero-sum game. In this work, we show through several theoretical and numerical results that indeed GAN zero-sum games may not have any local Nash equilibria. To characterize an equilibrium notion applicable to GANs, we consider the equilibrium of a new zero-sum game with an objective function given by a proximal operator applied to the original objective, a solution we call the \emph{proximal equilibrium}. Unlike the Nash equilibrium, the proximal equilibrium captures the sequential nature of GANs, in which the generator moves first followed by the discriminator. We prove that the optimal generative model in Wasserstein GAN problems provides a proximal equilibrium. Inspired by these results, we propose a new approach, which we call \emph{proximal training}, for solving GAN problems. We discuss several numerical experiments demonstrating the existence of proximal equilibrium solutions in GAN minimax problems.
\end{abstract}

\section{Introduction}

Since their introduction in \cite{goodfellow2014generative}, generative adversarial networks (GANs) have gained great success in many tasks of learning the distribution of observed samples. Unlike the traditional approaches to distribution learning, GANs view the learning problem as a zero-sum game between the following two players: 1) generator $G$ aiming to generate real-like samples from a random noise input, 2) discriminator $D$ trying to distinguish $G$'s generated samples from real training data. This game is commonly formulated through a minimax optimization problem as follows:
\begin{equation}\label{Eq: Intro_GAN_1}
    \min_{G\in \mathcal{G}}\; \max_{D \in \mathcal{D}}\; V(G,D).
\end{equation}
Here, $\mathcal{G}$ and $\mathcal{D}$ are respectively the generator and discriminator function sets, commonly chosen as two deep neural nets, and $V(G,D)$ denotes the minimax objective for generator $G$ and discriminator $D$ capturing how dissimilar the generated samples and training data are. 

GAN optimization problems are commonly solved by alternating gradient methods, which under proper regularization have resulted in state-of-the-art generative models for various benchmark datasets. However, GAN minimax optimization has led to several theoretical and empirical challenges in the machine learning literature. Training GANs is widely known as a challenging optimization task requiring an exhaustive hyper-parameter and architecture search and demonstrating an unstable behavior. While a few regularization schemes have achieved empirical success in training GANs \cite{salimans2016improved,arjovsky2017wasserstein,gulrajani2017improved,miyato2018spectral}, still little is known about the conditions under which GAN minimax optimization can be successfully solved by first-order optimization methods. 

To understand the minimax optimization in GANs, one needs to first answer the following question: What is the proper notion of equilibrium in the GAN zero-sum game? In other words, what are the optimality criteria in the GAN's minimax optimization problem?  A classical notion of equilibrium in the game theory literature is the \emph{Nash equilibrium}, a state in which no player can raise its individual gain by choosing a different strategy. According to this definition, a Nash equilibrium $(G^*,D^*)$ for the GAN  minimax problem \eqref{Eq: Intro_GAN_1} must satisfy the following for every $G\in\mathcal{G}$ and $D\in\mathcal{D}$:
\begin{equation}
 V(G^*,D)\, \le  \, V(G^*,D^*) \,  \le \,  V(G,D^*).    
\end{equation}
As a well-known result, for a generator $G$ expressive enough to reproduce the distribution of observed samples, Nash equilibrium exists for the generator producing the data distribution \cite{goodfellow2016deep}. However, such a Nash equilibrium would be of little interest from a learning perspective, since the trained generator merely overfits the empirical distribution of training samples \cite{arora2017generalization}. More importantly, state-of-the-art GAN architectures \cite{gulrajani2017improved, miyato2018spectral,zhang2018self,brock2018large} commonly restrict the generator function through various means of regularization such as batch or spectral normalization. Such regularization mechanisms do not allow the generator to produce the empirical distribution of observed data-points. Since the realizability assumption does not apply to such regularized GANs, the existence of Nash equilibria will not be guaranteed in their minimax problems.  

The above discussion motivates studying the equilibrium of GAN zero-sum games in the \textit{non-realizable settings} where the generator cannot express the empirical distribution of training data. Here, a natural question is whether a Nash equilibrium still exists for the GAN minimax problem. In this work, we focus on this question and demonstrate through several theoretical and numerical results that: 
\begin{itemize}[leftmargin=*]
\centering
  \item[] \normalsize\emph{Nash equilibrium may not exist in GAN zero-sum games.}
\end{itemize}
We provide theoretical examples of well-known GAN formulations including the vanilla GAN \cite{goodfellow2014generative}, Wasserstein GAN (WGAN) \cite{arjovsky2017wasserstein}, $f$-GAN \cite{nowozin2016f}, and the second-order Wasserstein GAN (W2GAN) \cite{feizi2017understanding} where no local Nash equilibria exist in their minimax optimization problems. We further perform numerical experiments on widely-used GAN architectures which suggest that an empirically successful GAN training may converge to non-Nash equilibrium solutions.

Next, we focus on characterizing a new notion of equilibrium for GAN problems. To achieve this goal, we consider the Nash equilibrium of a new zero-sum game where the objective function is given by the following proximal operator applied to the minimax objective $V(G,D)$ with respect to a norm on discriminator functions:
\begin{equation}\label{Def: intro proximal op}
V^{\operatorname{prox}}(G,D) \coloneqq \max_{\widetilde{D}\in\mathcal{D}}\: V(G,\widetilde{D}) - \bigl\Vert \widetilde{D}- D\bigr\Vert^2.   
\end{equation}
We refer to the Nash equilibrium of the new zero-sum game as the \emph{proximal equilibrium}. Given the inherent sequential nature of GAN problems where the generator moves first followed by the discriminator, we consider a Stackelberg game for its representation and focus on the subgame perfect equilibrium (SPE) of the game as the right notion of equilibrium for such problems \cite{jin2019minmax}. We prove that the proximal equilibrium of Wasserstein GANs provides an SPE for the GAN problem. This result applies to both the first-order and second-order Wasserstein GANs. In these cases, we show a proximal equilibrium exists for the optimal generator minimizing the distance to the data distribution.

Inspired by these theoretical results, we propose a proximal approach for training GANs, which we call \emph{proximal training}, by changing the original minimax objective to the proximal objective in \eqref{Def: intro proximal op}. In addition to preserving the optimal solution to the GAN minimax problem, proximal training can further enjoy the existence of Nash equilibrium solutions in the new minimax objective. We discuss numerical results supporting the proximal training approach and the role of proximal equilibrium solutions in various GAN problems. 

\section{Related Work}

Understanding the minimax optimization in modern machine learning applications including GANs has been a subject of great interest in the machine learning literature. A large body of recent works \cite{daskalakis2017training,nouiehed2019solving,mokhtari2019unified,thekumparampil2019efficient,zhang2019policy,mazumdar2019finding,fiez2019convergence,wang2019convergence,lin2019gradient} have analyzed the convergence properties of first-order optimization methods in solving different classes of minimax games. 

In a related work, \cite{jin2019minmax} proposes a new notion of local optimality, called \emph{local minimax}, designed for general sequential machine learning games. Compared to the notion of local minimax, the proximal equilibrium proposed in our work gives a notion of global optimality, which as we show directly applies to Wasserstein GANs. \cite{jin2019minmax} also provides examples of minimax problems where Nash equilibria do not exist; however, the examples do not represent GAN minimax problems. Some recent works \cite{lin2019gradient,lei2019sgd,wang2020on} have analyzed the convergence of different optimization methods to local minimax solutions. 

In another related work, \cite{daskalakis2018limit} analyzes the stable points of the gradient descent ascent (GDA) and optimistic GDA \cite{daskalakis2017training} algorithms, proving that they will give strict supersets of the local saddle points. Regarding the stability of GAN algorithms, \cite{nagarajan2017gradient} proves that the GDA algorithm will be locally stable for the vanilla and regularized Wasserstein GAN problems. \cite{feizi2017understanding} shows the GDA algorithm is globally stable for W2GANs with linear generator and quadratic discriminator functions. 

Regarding the equilibrium in GANs, \cite{arora2017generalization} studies the Nash equilibrium of GAN minimax games in realizable settings.
Also, \cite{arora2017generalization,hsieh2018finding} develop methods for finding mixed strategy Nash equilibria. On the other hand, our results focus on the pure strategies in non-realizable settings. \cite{fedus2017many} empirically studies the equilibrium of GAN problems regularized via the gradient penalty, reporting positive results on the stability of regularized GANs. However, our focus is on the existence of pure Nash equilibrium solutions. \cite{mroueh2017sobolev} suggests a moment matching GAN formulation using the Sobolev norm. As a different direction, we use the Sobolev norm to analyze equilibrium in GANs.   
Finally, developing GAN architectures with improved equilibrium and stability properties has been studied in several recent works \cite{salimans2016improved, berthelot2017began,heusel2017gans,mescheder2017numerics,roth2017stabilizing,kodali2017convergence, mescheder2018training,zhou2019lipschitz}.

\section{An Initial Experiment on Equilibrium in GANs}\label{Starting Exp on GANs}

To examine whether the Nash equilibrium exists in GAN problems empirically, we performed a simple numerical experiment. In this experiment, we applied three standard GAN implementations including the Wasserstein GAN with weight-clipping (WGAN-WC) \cite{arjovsky2017wasserstein}, the improved Wasserstein GAN with gradient penalty (WGAN-GP) \cite{gulrajani2017improved}, and the spectrally-normalized vanilla GAN (SN-GAN) \cite{miyato2018spectral}, to the two benchmark MNIST \cite{lecun1998mnist} and CelebA \cite{liu2015faceattributes} databases. We used the convolutional architecture of the DC-GAN \cite{radford2015unsupervised} optimized with the Adam \cite{kingma2014adam} or RMSprop \cite{hinton2012neural} (only for WGAN-WC) optimizers. 

We performed each of the GAN experiments for 200,000 generator iterations to reach $(G_{{\boldsymbol{\theta}}_{\operatorname{final}}},$ $D_{{\mathbf{w}}_{\operatorname{final}}})$ with ${\boldsymbol{\theta}}_{\operatorname{final}}$ and ${\mathbf{w}}_{\operatorname{final}}$ denoting the trained generator and discriminator parameters at the end of the 200,000 iterations. Our goal is to examine whether the solution pair $(G_{{\boldsymbol{\theta}}_{\operatorname{final}}},D_{{\mathbf{w}}_{\operatorname{final}}})$ represents a Nash equilibrium or not. To do this, we fixed the trained discriminator and kept optimizing the generator, i.e. continuing optimizing the generator $G_{\boldsymbol{\theta}}$ without changing the discriminator $D_{{\mathbf{w}}_{\operatorname{final}}}$. Here we solved the following optimization problem initialized at $\boldsymbol{\theta}^{(0)}={\boldsymbol{\theta}}_{\operatorname{final}}$ using the default first-order optimizer for the generator function for 10,000 iterations:
\begin{equation}
    \min_{\boldsymbol{\theta}}\: V(G_{\boldsymbol{\theta}} , D_{{\mathbf{w}}_{\operatorname{final}}}).
\end{equation}
If the pair $(G_{{\boldsymbol{\theta}}_{\operatorname{final}}},D_{{\mathbf{w}}_{\operatorname{final}}})$ was in fact a Nash equilibrium, it would give a local saddle point to the minimax optimization and the above optimization could not make the objective any smaller than its initial value. Also, the image samples generated by the generator $G_{\boldsymbol{\theta}}$ should have improved or at least preserved their initial quality during this optimization, since the discriminator $D_{{\mathbf{w}}_{\operatorname{final}}}$ would be the optimal discriminator against all generator functions.

\begin{figure}
\centering
\begin{subfigure}{0.49\textwidth}
\centering
\includegraphics[width=0.98\linewidth]{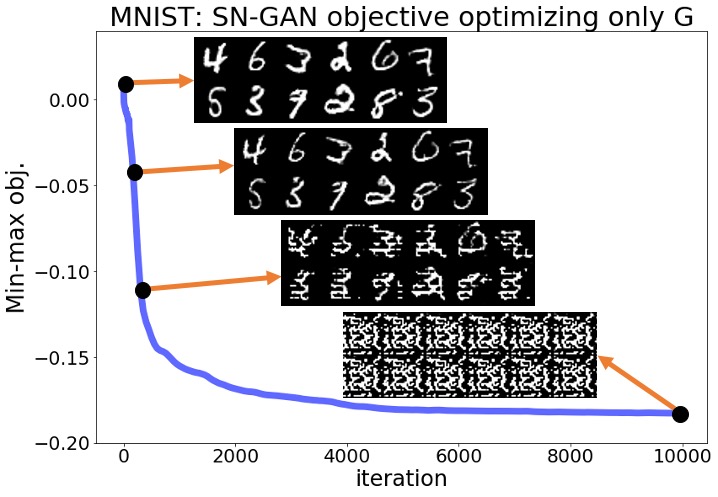}
\caption{Results on the MNIST data} 
\label{fig:mnist_eq}
\end{subfigure}
\begin{subfigure}{.49\textwidth}
\centering
\includegraphics[width=0.98\linewidth]{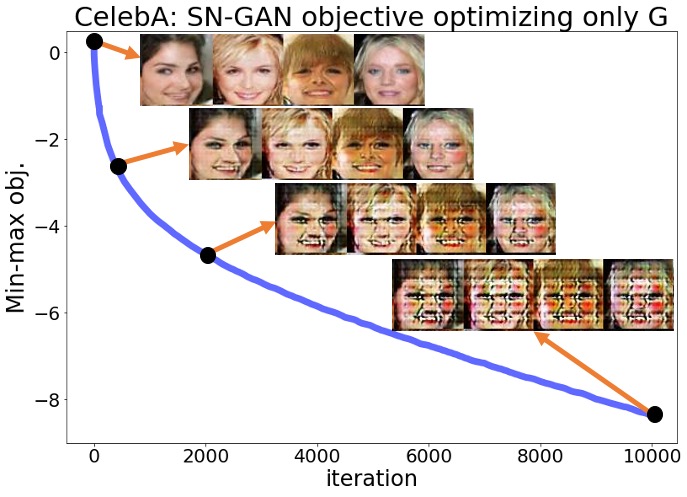}
\caption{Results on the CelebA data}
\label{fig:celeba_eq}
\end{subfigure}
\caption{Optimizing the SN-GAN's generator objective without changing the trained discriminator on the MNIST and CelebA data. Both the SN-GAN objective and samples' quality were decreasing during the optimization.} 
\end{figure}

Despite the above predictions, we observed that none of the mentioned statements hold in reality for any of the six experiments with the three standard GAN implementations and the two datasets. The optimization objective decreased rapidly from the beginning of the optimization, and the pictures sampled from the generator completely lost their quality over this optimization. Figures \ref{fig:mnist_eq}, \ref{fig:celeba_eq} show the objective for the SN-GAN experiments over the 10,000 steps of the above optimization. These figures also demonstrate the SN-GAN generated samples before and during the optimization, which shows the significant drop in the quality of generated pictures. We defer the results for the WGAN-WC and WGAN-GP problems to the Appendix.

The results of the above experiments show that practical GAN experiments may not converge to local Nash equilibrium solutions. After fixing the trained discriminator, the trained generator can be further optimized using a first-order optimization method to reach smaller values of the generator objective. More importantly, this optimization not only  does not improve the quality of the generator's output samples, but also totally disturbs the trained generator. As demonstrated in these experiments, simultaneous optimization of the two players is in fact necessary for the proper convergence and stability behavior in GAN minimax optimization. The above experiments suggest that practical GAN solutions are not local Nash equilibrium. In the upcoming sections, we review some standard GAN formulations and then show that there are examples of GAN minimax problems for which no Nash equilibrium exists. 
Those theoretical results will further support our observations in the above experiments.

\section{Review of GAN Formulations}\label{Section: GAN formulations}
\subsection{Vanilla GAN \& $f$-GAN}
Consider samples $\mathbf{x}_1,\ldots,\mathbf{x}_n$ observed independently from distribution $P_{\mathbf{X}}$. Our goal is to find a generator function $G\in\mathcal{G}$ where $G(\mathbf{Z})$ maps a random noise input $\mathbf{Z}$ from a known $P_{\mathbf{Z}}$ to an output $G(\mathbf{Z})$ distributed as $P_{\mathbf{X}}$, i.e., we aim to match the probability distributions $P_{G(\mathbf{Z})}$ and $P_\mathbf{X}$. To find such a generator function, \cite{goodfellow2014generative} proposes the following minimax problem which is commonly referred to as the \emph{vanilla GAN} problem:
\begin{equation}\label{Vanilla GAN: formulation}
    \min_{G\in\mathcal{G}}\: \max_{D\in\mathcal{D}}\: \mathbb{E}\bigl[\log(D(\mathbf{X}))\bigr] + \mathbb{E}\bigl[\log(1-D(G(\mathbf{Z})))\bigr].
\end{equation}
Here $\mathcal{G}$ and $\mathcal{D}$ represent the set of generator and discriminator functions, respectively. In this formulation, the discriminator is optimized to map real samples from $P_\mathbf{X}$ to larger values than the values assigned to generated samples from $P_{G(\mathbf{Z})}$. 

As shown in \cite{goodfellow2014generative}, the above minimax problem for an unconstrained $\mathcal{D}$ containing all real-valued functions reduces to the following divergence minimization problem:
\begin{equation}\label{GAN_interprettaion: JSD}
    \min_{G\in \mathcal{G}}\: \operatorname{JSD}\bigl(P_{\mathbf{X}},P_{G(\mathbf{Z})}\bigr),
\end{equation}
where $\operatorname{JSD}$ denotes the Jensen-Shannon (JS) divergence defined in terms of KL-divergence as
\begin{equation*}
    \operatorname{JSD}(P,Q) \coloneqq  \frac{1}{2}\operatorname{KL}\bigl(P,\frac{P+Q}{2}\bigr)\,+\,\frac{1}{2}\operatorname{KL}\bigl(Q,\frac{P+Q}{2}\bigr).
\end{equation*}
$f$-GANs extend the vanilla GAN problem by generalizing the JS-divergence to a general $f$-divergence. For a convex function $f:\mathbb{R}\rightarrow \mathbb{R}$ with $f(1)=0$, the $f$-divergence $d_f$ corresponding to $f$ is defined as
\begin{equation}
    d_{f}(P,Q) \coloneqq \int p(\mathbf{x})f\bigl(\frac{q(\mathbf{x})}{p(\mathbf{x})}\bigr)\diff \mathbf{x}.
\end{equation}
Notice that the JS-divergence is a special case of $f$-divergence with $f_{\operatorname{JSD}}(t) = t\log t - (t+1)\log\frac{t+1}{2}$. \cite{nowozin2016f} shows that generalizing the divergence minimization \eqref{GAN_interprettaion: JSD} to minimizing a $f$-divergence results in the following minimax problem called $f$-GAN:
\begin{equation}\label{$f$-GAN: Problem}
    \min_{G\in\mathcal{G}}\: \max_{D\in\mathcal{D}}\: \mathbb{E}\bigl[D(\mathbf{X})] - \mathbb{E}\bigl[f^*\bigl(D(G(\mathbf{Z}))\bigr)\bigr],    
\end{equation}
where $f^*$ denotes the Fenchel-conjugate to $f$ defined as $f^*(u)=\sup_t ut-f(t)$. The space $\mathcal{D}$ implied by the f-divergence minimization will be the set of all functions, but a similar interpretation further applies to a constrained $\mathcal{D}$ \cite{liu2017approximation,farnia2018convex}. Several examples of $f$-GANs have been formulated and discussed in \cite{nowozin2016f}.

\subsection{Wasserstein GANs}
To resolve GAN training issues, \cite{arjovsky2017wasserstein} proposes to formulate a GAN problem by minimizing the optimal transport costs which unlike $f$-divergences change continuously with the input distributions. Given a transportation cost $c(\mathbf{x},\mathbf{x}')$ for transporting $\mathbf{x}$ to $\mathbf{x}'$, the optimal transport cost $W_c$ is defined as
\begin{equation}
    W_c(P,Q) = \inf_{M \in \Pi(P,Q)}\: \mathbb{E}_M\bigl[ c(\mathbf{X},\mathbf{X}') \bigr]
\end{equation}
where $\Pi(P,Q)$ denotes the set of all joint distributions on $(\mathbf{X},\mathbf{X}')$ with $\mathbf{X},\, \mathbf{X}'$ marginally distributed as $P,\, Q$, respectively. An important special case is the first-order Wasserstein distance ($W_1$-distance) corresponding to $c(\mathbf{X},\mathbf{X}')=\Vert \mathbf{X}- \mathbf{X}'\Vert$. In this special case, the Kantorovich-Rubinstein duality shows
\begin{equation}
W_1(P,Q) = \max_{\Vert D\Vert_{\operatorname{Lip}}\le 1}\, \mathbb{E}_P[D(\mathbf{X})] - \mathbb{E}_Q[D(\mathbf{X})].    
\end{equation}
Here $\mathbb{E}_P[\cdot]$ denotes the expected value with respect to distribution $P$ and $\Vert D \Vert_{\operatorname{Lip}}$ denotes the Lipschitz constant of function $D$ which is defined as the smallest $L$ satisfying $\vert D(\mathbf{x}) - D(\mathbf{x}') \vert \le L\,\Vert \mathbf{x} - \mathbf{x}'\Vert $ for every $\mathbf{x},\mathbf{x}'$. Formulating a GAN problem minimizing the $W_1$-distance, \cite{arjovsky2017wasserstein} states the Wasserstein GAN (WGAN) problem as follows:
\begin{equation}\label{WGAN: Arjovsky}
    \min_{G\in\mathcal{G}}\: \max_{\Vert D\Vert_{\operatorname{Lip}} \le 1}\: \mathbb{E}\bigl[D(\mathbf{X})] - \mathbb{E}\bigl[D(G(\mathbf{Z}))\bigr].   
\end{equation}
The above Wasserstein GAN problem can be generalized to a general optimal transport cost with arbitrary cost function $c(\mathbf{x},\mathbf{x}')$. The generalization is as follows:
\begin{equation}\label{WGAN: General cost}
    \min_{G\in\mathcal{G}}\: \max_{D\, \operatorname{c-concave}}\: \mathbb{E}\bigl[D(\mathbf{X})] - \mathbb{E}\bigl[D^c(G(\mathbf{Z}))\bigr],    
\end{equation}
where the c-transform is defined as $D^c(\mathbf{x})=\sup_{\mathbf{x}'}\, D(\mathbf{x}')-c(\mathbf{x},\mathbf{x}')$ and a function $D$ is called c-concave if it is the c-transform of some valid function. In particular, the optimal transport GAN formulation with the quadratic cost $c(\mathbf{x},\mathbf{x}')=\Vert \mathbf{x} - \mathbf{x}' \Vert_2^2$ results in the second-order Wasserstein GAN (W2GAN) problem which has been studied in several recent works \cite{feizi2017understanding,salimans2018improving,sanjabi2018convergence,taghvaei20192}.

\section{Existence of Nash Equilibrium Solutions in GANs}
Consider a general GAN minimax problem \eqref{Eq: Intro_GAN_1} with a minimax objective $V(G,D)$. As discussed in the previous section, the optimal generator $G^*\in\mathcal{G}$ is defined to minimize the GAN's target divergence to the data distribution. The following proposition is a well-known result regarding the Nash equilibrium of the GAN game in realizable settings where there exists a generator $G\in\mathcal{G}$ producing the data distribution.
\begin{prop}\label{Prop: Nash_exists}
Assume that generator $G^*\in\mathcal{G}$ results in the distribution of data, i.e., we have $P_{G^*(\mathbf{Z})} = P_{\mathbf{X}}$. Then, for each of the GAN problems discussed in Section \ref{Section: GAN formulations} there exists a constant discriminator function $D_{\operatorname{constant}}$ which together with $G^*$ results in a Nash equilibrium to the GAN game, and hence satisfies the following for every $G\in\mathcal{G}$ and $D\in\mathcal{D}$:
\begin{equation*}
    V(G^*,D)\, \le \, V(G^*,D_{\operatorname{constant}}) \, \le \, V(G,D_{\operatorname{constant}}).
\end{equation*}
\end{prop}
\begin{proof}
This proposition is well-known for the vanilla GAN \cite{goodfellow2016nips}. In the Appendix, we provide a proof for general $f$-GANs and Wasserstein GANs.
\end{proof}
The above proposition shows that in a realizable setting with a generator function generating the distribution of observed samples, a Nash equilibrium exists for that optimal generator. However, the realizability assumption in this proposition does not always hold in real GAN experiments. For example, in the GAN experiments discussed in Section \ref{Starting Exp on GANs}, we observed that the divergence estimate never reached the zero value because of regularizing the generator function. Therefore, the Nash equilibrium described in Proposition \ref{Prop: Nash_exists} does not apply to the trained generator and discriminator in such GAN experiments. 

Here, we address the question of the existence of Nash equilibrium solutions for non-realizable settings, where no generator $G\in\mathcal{G}$ can produce the data distribution. Do Nash equilibria always exist in non-realizable GAN zero-sum games? The following theorem shows that the answer is in general no. Note that $\sigma_{\max}(\cdot)$ in this theorem denotes the maximum singular value, i.e., the spectral norm.
\begin{thm}\label{Thm: no Nash}
Consider a GAN minimax problem for learning a normally distributed $\mathbf{X}\sim \mathcal{N}(\mathbf{0},\sigma^2 I)$ with zero mean and scalar covariance matrix where $\sigma > 1$. In the GAN formulation, we use a linear generator function $G(\mathbf{z})=\mathbf{W}\mathbf{z}+\mathbf{u}$ where the weight matrix $\mathbf{W}$ is spectrally-regularized to satisfy $\sigma_{\max}( \mathbf{W})\le 1$. Suppose that the Gaussian latent vector is normally distributed as $\mathbf{Z}\sim \mathcal{N}(\mathbf{0},I)$ with zero mean and identity covariance matrix. Then,
\begin{itemize}[leftmargin=*]
    \item For the $f$-GAN problem corresponding to an $f$ with non-decreasing $t^2f''(t)$ over $t\in (0,+\infty)$ and an unconstrained discriminator $D$ where the dimensions of data $\mathbf{X}$ and latent $\mathbf{Z}$ match, the f-GAN minimax problem has no Nash equilibrium solutions.
    \item For the W2GAN problem with discriminator $D$ trained over $c$-concave functions, where $c$ is the quadratic cost, the W2GAN minimax problem has no Nash equilibrium solutions. Also, given a quadratic discriminator $ D(\mathbf{x})=\mathbf{x}^TA\mathbf{x}+\mathbf{b}^T\mathbf{x}$ parameterized by $A,\mathbf{b}$, the W2GAN problem has no \textbf{local} Nash equilibria.
    \item For the WGAN problem with $1$-dimensional $X,Z$ and a discriminator $D$ trained over 1-Lipschitz functions, the WGAN minimax problem has no Nash equilibria.
\end{itemize}
\end{thm}
\begin{proof}
We defer the proof to the Appendix. Note that the condition on the $f$-GAN holds for all $f$-GAN examples in \cite{nowozin2016f} including the vanilla GAN.
\end{proof} 
The above theorem shows that under the stated assumptions the GAN zero-sum game does not have Nash equilibrium solutions. Consequently, the optimal divergence-minimizing generative model does not result in a Nash equilibrium. In contrast to Theorem \ref{Thm: no Nash}, the following remark shows that the GAN zero-sum game in a non-realizable case may have Nash equilibrium solutions, of course if Theorem \ref{Thm: no Nash}'s assumptions do not hold.
\begin{remark}
Consider the same setting as in Theorem \ref{Thm: no Nash}. However, unlike Theorem \ref{Thm: no Nash} suppose that $\sigma < 1$ and $\sigma_{\min}( \mathbf{W} ) \ge 1$ where $\sigma_{\min}$ stands for the minimum singular value. Then, for the WGAN and W2GAN problems described in Theorem \ref{Thm: no Nash}, the Wasserstein distance-minimizing generator results in a Nash equilibrium.
\end{remark}
\begin{proof}
We defer the proof to the Appendix.
\end{proof}
The above remark explains that the phenomenon shown in Theorem \ref{Thm: no Nash} does not always hold in non-realizable GAN settings. As a result, we need other notions of equilibrium which consistently explain optimality in GAN games.    

\section{Proximal Equilibrium: A Relaxation of Nash Equilibrium}
To define a proper notion of equilibrium for GANs, note that due to the sequential nature of GAN games the equilibrium notion should be flexible to allow to some extent the optimization of the discriminator around the equilibrium solution. This property is in fact consistent with the stability feature observed for the first-order GAN training methods where the alternating first-order method stabilizes around a certain solution. To this end, we consider the following objective for a GAN problem with minimax objective $V(G,D)$: 
\begin{equation}\label{Definition: Proximal func}
    V_{\lambda}^{\operatorname{prox}}(G,D) \, \coloneqq \, \max_{ \widetilde{D} \in \mathcal{D}}\, V(G,\widetilde{D})-\frac{\lambda}{2}\bigl\Vert \widetilde{D} - D\bigr\Vert^2. 
\end{equation}
The above definition represents the application of a proximal operator to $V(G,D)$, which further optimizes the original objective in the proximity of discriminator $D$. To keep the $\widetilde{D}$ function variable close to $D$, we penalize the distance among the two functions in the proximal optimization. Here the distance is measured using a norm $\Vert\cdot\Vert$ on the discriminator function space. 

To extend the notion of Nash equilibrium to general minimax problems, we propose considering the Nash equilibria of the defined $V_{\lambda}^{\operatorname{prox}}(G,D)$.
\begin{mydef}
We call $(G^*,D^*)$ a $\lambda$-proximal equilibrium for $V(G,D)$ if it represents a Nash equilibrium for
$V_{\lambda}^{\operatorname{prox}}(G,D)$ 
, i.e. for every $G\in\mathcal{G}$ and $D\in\mathcal{D}$
\begin{equation}
  V_{\lambda}^{\operatorname{prox}}(G^*,D) \le V_{\lambda}^{\operatorname{prox}}(G^*,D^*) \le V_{\lambda}^{\operatorname{prox}}(G,D^*).  
\end{equation}
\end{mydef}
The next proposition provides necessary and sufficient conditions in terms of the original objective $V(G,D)$ for the proximal equilibrium solutions.
\begin{prop}\label{Prop: 2}
$(G^*,D^*)$ is a $\lambda$-proximal equilibrium if and only if for every $G\in\mathcal{G}$ and $D\in\mathcal{D}$ we have
    \begin{align*}
        V(G^*,D)\le V(G^*,D^*)  \le  \max_{\widetilde{D}\in \mathcal{D}} V(G,\widetilde{D}) - \frac{\lambda}{2}\bigl\Vert \widetilde{D} - D^* \bigr\Vert^2
    \end{align*}
Therefore, if $(G^*,D^*)$ is a $\lambda$-proximal equilibrium it will give a global minimax solution, i.e., $G^*\in\mathcal{G}$ minimizes the worst-case objective, $\max_{D\in\mathcal{D}}V(G,D)$, with $D^*$ being its optimal solution. 
\end{prop}
\begin{proof}
We defer the proof to the Appendix.
\end{proof}
The following result shows the proximal equilibria provide a hierarchy of equilibrium solutions for different $\lambda$ values.
\begin{prop} 
 Define $\operatorname{PE}_{\lambda}(V)$ to be the set of the $\lambda$-proximal equilibria for $V(G,D)$. Then, if $\lambda_1 \le \lambda_2$,
\begin{equation}
 \operatorname{PE}_{\lambda_2}(V) \subseteq \operatorname{PE}_{\lambda_1}(V).
\end{equation}
\end{prop}
\begin{proof}
We defer the proof to the Appendix.
\end{proof}
Note that as $\lambda$ approaches infinity, $V_{\lambda}^{\operatorname{prox}}(G,D)$ tends to the original $V(G,D)$, implying that $\operatorname{PE}_{\lambda=+\infty}(V) $ is the set of $V(G,D)$'s Nash equilibria. In contrast, for $\lambda=0$ the proximal objective becomes the worst-case objective $\max_{D\in\mathcal{D}}V(G,D)$. As a result, $\operatorname{PE}_{\lambda=0}(V)$ is the set of global minimax solutions described in Proposition \ref{Prop: 2}. 

Concerning the proximal optimization problem in \eqref{Definition: Proximal func}, the following proposition shows that if the original minimax objective is a smooth function of the discriminator parameters, the proximal optimization can be solved efficiently and therefore one can efficiently compute the gradient of the proximal objective. 
\begin{prop}\label{Proposition: Gradient of Prox}
Consider the maximization problem in the definition of proximal objective \eqref{Definition: Proximal func} where generator $G_{\boldsymbol{\theta}}$ and discriminator $D_{\mathbf{w}}$ are parameterized by vectors $\boldsymbol{\theta},\,\mathbf{w}$, respectively. Suppose that
\begin{itemize}[leftmargin=*]
    \item For the considered discriminator norm $\Vert \cdot \Vert$, $\Vert D_{\mathbf{w}} -D \Vert^2$ is $\eta_1$-strongly convex in $\mathbf{w}$ for any function $D$, i.e. for any $\mathbf{w},\mathbf{w}', D$:
    \begin{equation*}
        \bigl\Vert \nabla_{\mathbf{w}} \Vert D_{\mathbf{w}} - D \Vert^2 \, - \, \nabla_{\mathbf{w}} \Vert D_{\mathbf{w}'} - D \Vert^2  \bigr\Vert_2 \, \ge \, \eta_1 \bigl\Vert \mathbf{w} - \mathbf{w}' \bigr\Vert_2,    
    \end{equation*}
    \item For every $G_{\boldsymbol{\theta}}$, The GAN minimax objective $V(G_{\boldsymbol{\theta}},D_{\mathbf{w}})$ is $\eta_2$-smooth in $\mathbf{w}$, i.e. i.e. for any $\mathbf{w},\mathbf{w}', \boldsymbol{\theta}$:
    \begin{equation*}
        \bigl\Vert \nabla_{\mathbf{w}} V(G_{\boldsymbol{\theta}},D_{\mathbf{w}}) \, - \, \nabla_{\mathbf{w}} V(G_{\boldsymbol{\theta}},D_{\mathbf{w}'})  \bigr\Vert_2 \, \le \, \eta_2 \Vert \mathbf{w} - \mathbf{w}' \Vert_2.    
    \end{equation*}
\end{itemize}
Under the above assumptions, if $\eta_2 < \frac{\lambda\eta_1}{2}$, the maximization objective in \eqref{Definition: Proximal func} is $\frac{\lambda\eta_1}{2} - \eta_2$-strongly concave. Then, the maximization problem has a unique solution $\mathbf{w}^*$ and if $V(G_{\boldsymbol{\theta}},D_{\mathbf{w}})$ is differentiable with respect to $\theta$ we have
\begin{equation}
    \nabla_{\theta} V_{\lambda}^{\operatorname{prox}}(G_{\boldsymbol{\theta}},D_{\mathbf{w}})  = \nabla_{\theta} V(G_{\boldsymbol{\theta}},D_{\mathbf{w}^*}).
\end{equation}
\end{prop}
\begin{proof}
We defer the proof to the Appendix.
\end{proof}
The above proposition suggests that under the mentioned assumptions, one can efficiently compute the optimal solution to the proximal maximization through a first-order optimization method. The assumptions require the smoothness of the GAN minimax objective with respect to the discriminator parameters, which can be imposed by applying norm-based regularization tools to neural network discriminators.

\section{Proximal Equilibrium in Wasserstein GANs}

As shown earlier, GAN minimax games may not have any Nash equilibria in non-realizable settings. As a result, we seek for a different notion of equilibrium which remains applicable to GAN problems. Here, we show the proposed proximal equilibrium provides such an equilibrium notion for Wasserstein GAN problems.

To define a proper proximal operator for defining proximal equilibria in Wasserstein GAN problems, we use the second-order Sobolev semi-norm averaged over the underlying distribution of data. Given the underlying distribution $P_\mathbf{X}$, we define the Sobolev semi-norm as
\begin{align}\label{Definition: Sobolev Norm}
    \big\Vert D \big\Vert_{\dot{H}^1}\, \coloneqq \, \sqrt{\, \mathbb{E}_{P_{\mathbf{X}}}\left[ \big\Vert \nabla_{\mathbf{x}} D(\mathbf{X}) \big\Vert^2_2\, \right]}\,.
\end{align}
The above semi-norm is induced by the following semi-inner product and therefore leads to a semi-Hilbert space of functions:
\begin{equation}
    \langle D_1,D_2{\rangle}_{\dot{H}^1} \coloneqq \mathbb{E}_{P_{\mathbf{X}}}\left[\,   \nabla D_1(\mathbf{X})^T \nabla D_2(\mathbf{X}) \, \right].
\end{equation}

Throughout our discussion, we consider a parameterized set of generators $\mathcal{G}=\{G_{\boldsymbol{\theta}}:\, \boldsymbol{\theta}\in\Theta  \}$. For a GAN minimax objective $V(G,D)$, we define $D^{\boldsymbol{\theta}}$ to be the optimal discriminator function for the parameterized generator $G_{\boldsymbol{\theta}}$:
\begin{equation}
   D^{\boldsymbol{\theta}} \coloneqq  \underset{D\in \mathcal{D}}{\arg\!\max}\: V(G_{\boldsymbol{\theta}},D). 
\end{equation}
The following theorem shows that the Wasserstein distance-minimizing generator function in the second-order Wasserstein GAN problem satisfies the conditions of a proximal equilibrium based on the Sobolev semi-norm defined in \eqref{Definition: Sobolev Norm}.
\begin{thm}\label{Thm: Prox for 2nd}
Consider the second-order Wasserstein GAN problem \eqref{WGAN: General cost} with a quadratic cost $c(\mathbf{x},\mathbf{x}')=\frac{\eta}{2}\Vert \mathbf{x} - \mathbf{x}'\Vert_2^2 $. Suppose that the set of optimal discriminators $\{D^{\boldsymbol{\theta}}:\, \boldsymbol{\theta}\in\Theta  \}$ is convex.  
Then, $(G_{\boldsymbol{\theta}^*},D^{\boldsymbol{\theta}^*})$ for the Wasserstein distance-minimizing generator $G_{\boldsymbol{\theta}^*}\in \mathcal{G}$ will provide a $\frac{1}{\eta}$-proximal equilibrium with respect to the Sobolev norm in \eqref{Definition: Sobolev Norm}. 
\end{thm}
\begin{proof}
We defer the proof to the Appendix.
\end{proof}
The above theorem shows that while, as demonstrated in Theorem  \ref{Thm: no Nash}, the W2GAN problem may have no local  Nash equilibrium solutions, the proximal equilibrium exists for the W2GAN problem and holds at the Wasserstein-distance minimizing generator $G_{\boldsymbol{\theta}^*}$. The next theorem extends this result to the first-order Wasserstein GAN (WGAN) problem.
\begin{thm}\label{Thm: Prox for 1st}
Consider the WGAN problem \eqref{WGAN: Arjovsky} minimizing the first-order Wasserstein distance. For each $G_{\boldsymbol{\theta}}$, define $\alpha_{\boldsymbol{\theta}}:\mathbb{R}^d\rightarrow \mathbb{R}^{\ge 0}$ to be the magnitude of the resulted optimal transport map from $\mathbf{X}$ to $G_{\boldsymbol{\theta}}(\mathbf{Z})$, i.e. $\mathbf{X} - {\alpha}^2_{\boldsymbol{\theta}}(\mathbf{X})\nabla D^{\boldsymbol{\theta}}(\mathbf{X})$ shares the same distribution with $G_{\boldsymbol{\theta}}(\mathbf{Z})$.\footnote{Note that as shown in the proof such a mapping $\alpha_{\boldsymbol{\theta}}$ exists under mild regularity assumptions.} Given these definitions, assume that
\begin{itemize}[leftmargin=*]
    \item $\{ {\alpha}_{\boldsymbol{\theta}}(\cdot)\nabla D^{\boldsymbol{\theta}}(\cdot):\, \boldsymbol{\theta}\in\Theta \}$ is a convex set,
    \item for every $\mathbf{x}$ and $\boldsymbol{\theta}$, $\frac{\eta}{2}\le \alpha_{\boldsymbol{\theta}}^2(\mathbf{x})$ holds for  constant $\eta$.
\end{itemize}
Then, $(G_{\boldsymbol{\theta}^*},D^{\boldsymbol{\theta}^*})$ for the Wasserstein distance-minimizing generator function $G_{\boldsymbol{\theta}^*}$ provides an $\eta$-proximal equilibrium with respect to the Sobolev norm in \eqref{Definition: Sobolev Norm}. 
\end{thm}
\begin{proof}
We defer the proof to the Appendix.
\end{proof}
The above theorem shows that if the magnitude of optimal transport map is everywhere lower-bounded by $\frac{\lambda}{2}$, then the Wasserstein distance-minimizing generator in the WGAN problem yields a $\lambda$-proximal equilibrium. 

\section{Proximal Training}
As shown for Wasserstein GAN problems, given the defined Sobolev norm and a small enough $\lambda$ the proximal objective $V_{\lambda}^{\operatorname{prox}}(G,D)$ will possess a Nash equilibrium solution. This result motivates performing the minimax optimization for the proximal objective $V_{\lambda}^{\operatorname{prox}}(G,D)$ instead of the original objective $V(G,D)$. Therefore, we propose proximal training in which we solve the following minimax optimization problem:
\begin{equation}
    \min_{G_{\boldsymbol{\theta}}\in\mathcal{G}}\; \max_{D_{\mathbf{w}}\in\mathcal{D}}\;
    V_{\lambda}^{\operatorname{prox}}(G_{\boldsymbol{\theta}},D_{\mathbf{w}}),
\end{equation}
with the proximal operator defined according to the Sobolev norm in \eqref{Definition: Sobolev Norm}. 

In order to take the gradient of $V_{\lambda}^{\operatorname{prox}}(G_{\boldsymbol{\theta}},D_{\mathbf{w}})$ with respect to $\boldsymbol{\theta}$, Proposition \ref{Proposition: Gradient of Prox} suggests solving the proximal optimization followed by computing the gradient of the original objective $V(G_{\boldsymbol{\theta}},D_{\mathbf{w}^*})$ where the discriminator is parameterized with the optimal solution $\mathbf{w}^*$ to the proximal optimization. 

\begin{algorithm}[h]
   \caption{GAN Proximal Training}
   \label{alg:Proximal_Training}
\begin{algorithmic}
  \STATE {\bfseries Input:} data $\mathbf{x}_i$, size $n$
  \vspace{1.5mm}
  \STATE Initialize the parameters $\mathbf{w}^{(0)} , \boldsymbol{\theta}^{(0)}$
  \vspace{1.5mm}
  \FOR{$\text{\rm k}=0$ {\bfseries to} $\text{\rm MAX\_ITER}$}
  \vspace{2mm}
   \STATE $\gg$ Update $\mathbf{w}^{(k+1)}\, = \, \underset{\mathbf{w}}{\arg\!\max}\: V(G_{\boldsymbol{\theta}^{(k)}},D_{\mathbf{w}}) -  \frac{\lambda}{2n}\sum_{i=1}^n \Vert \nabla_{\mathbf{x}}D_{\mathbf{w}}(\mathbf{x}_i) - \nabla_{\mathbf{x}}D_{\mathbf{w}^{(k)}}(\mathbf{x}_i)  \Vert^2_{2}$. 
   \vspace{2mm}
   \STATE $\gg$ Update ${\boldsymbol{\theta}}^{(k+1)} = {\boldsymbol{\theta}}^{(k)} - \gamma_k \nabla_{\boldsymbol{\theta}} V(G_{\boldsymbol{\theta}^{(k)}},D_{\mathbf{w}^{(k+1)}}).$
   \vspace{1.5mm}
  \ENDFOR
\end{algorithmic}
\end{algorithm}

Algorithm~\ref{alg:Proximal_Training} summarizes the main two steps of proximal training. At every iteration, the discriminator is optimized with an additive Sobolev norm penalty forcing the discriminator to remain in the proximity of the current discriminator. Next, the generator is optimized using a gradient descent method with the gradient evaluated at the optimal discriminator solving the proximal optimization. The stepsize parameter $\gamma_k$ can be adaptively selected at every iteration $k$. In practice, we can solve the proximal maximization problem via a first-order optimization method for a certain number of iterations. Assuming the conditions of Proposition \ref{Proposition: Gradient of Prox} hold, the proximal optimization leads to the maximization of a strongly-concave objective which can be solved linearly fast through first-order optimization methods.

\section{Numerical Experiments}
To experiment the theoretical results of this work, we performed several experiments using the \cite{gulrajani2017improved}'s implementation of Wasserstein GANs with the code available at the paper's Github repository. In addition, we used the implementations of \cite{miyato2018spectral,farnia2018generalizable} for applying spectral regularization to the discriminator network. In the experiments, we used the DC-GAN 4-layer CNN architecture for both the discriminator and generator functions \cite{radford2015unsupervised} and ran each experiment for 200,000 generator iterations with 5 discriminator updates per generator update. We used the RMSprop optimzier \cite{hinton2012neural} for WGAN experiments with weight clipping or spectral normalization 
and the Adam optimizer \cite{kingma2014adam} 
for the other experiments.

\subsection{Proximal Equilibrium in Wasserstein and Lipschitz GANs}

\begin{figure}
\centering
\begin{subfigure}{0.49\textwidth}
\centering
\includegraphics[width=0.95\linewidth]{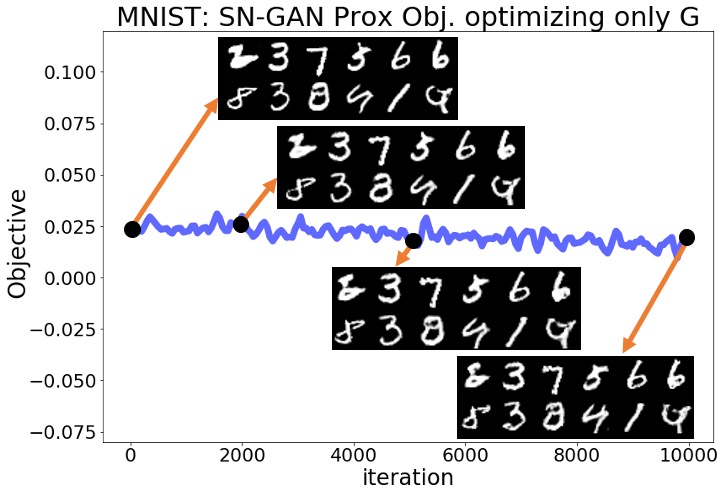}
\caption{Results on the MNIST data}
\label{fig:mnist_eq_pt}
\end{subfigure}
\begin{subfigure}{.49\textwidth}
\centering
\includegraphics[width=0.95\linewidth]{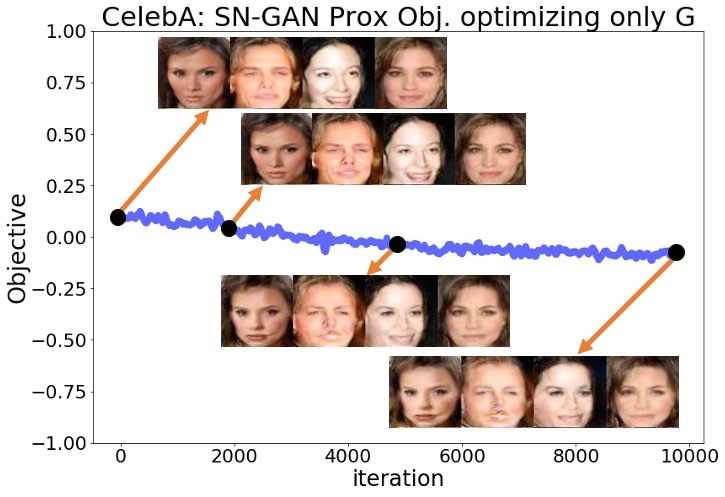}
\caption{Results on the CelebA data} 
\label{fig:celeba_eq_pt}
\end{subfigure}
\caption{Optimizing the proximal objective over the generator with a fixed discriminator on MNIST and CelebA datasets. The SN-GAN's objective and samples' quality were preserved during the optimization.} 
\end{figure}

We examined whether the solutions found by Wasserstein and Lipschitz vanilla GANs represent proximal equilibria. Toward this goal, we performed similar experiments to Section \ref{Starting Exp on GANs}'s experiments for the WGAN-WC \cite{arjovsky2017wasserstein}, WGAN-GP \cite{gulrajani2017improved}, and SN-GAN \cite{miyato2018spectral} problems over the MNIST and CelebA datasets. In Section \ref{Starting Exp on GANs}, we observed that after fixing the trained discriminator $D_{\mathbf{w}_{\operatorname{final}}}$ the GAN's minimax objective $V(G_{\boldsymbol{\theta}},D_{\mathbf{w}_{\operatorname{final}}})$ kept decreasing when we optimized only the generator $G_{\boldsymbol{\theta}}$. In the new experiments, we similarly fixed the trained discriminator $D_{\mathbf{w}_{\operatorname{final}}}$resulted from the 200,000 training iterations, but instead of optimizing the GAN minimax objective we optimized the \textit{proximal} objective defined by the norm \eqref{Definition: Sobolev Norm} with $\lambda=0.1$. Thus, we solved the following optimization problem initialized at ${\boldsymbol{\theta}}_{\operatorname{final}}$ which denotes the parameters of the trained generator:
\begin{equation}
    \min_{\boldsymbol{\theta}}\: V^{\operatorname{prox}}_{\lambda=0.1}(G_{\boldsymbol{\theta}},D_{\mathbf{w}_{\operatorname{final}}}).
\end{equation}
We computed the gradient of the above proximal objective by applying the Adam optimizer for $50$ steps to approximate the solution to the proximal optimization \eqref{Definition: Proximal func} which at every iteration was initialized at $\mathbf{w}_{\operatorname{final}}$. Figures \ref{fig:mnist_eq_pt} and \ref{fig:celeba_eq_pt} show that in the SN-GAN experiments the original minimax objective had only minor changes, compared to the results in Section \ref{Starting Exp on GANs}, and the quality of generated samples did not change significantly during the optimization. We defer the similar numerical results of the WGAN-WC and WGAN-GP experiments to the Appendix. These numerical results suggest that while Wasserstein and Lipschitz GANs may not converge to local Nash equilibrium solutions as shown in Section \ref{Starting Exp on GANs}, their found solutions can still represent a local proximal equilibrium.

 \subsection{Proximal Training Improves Lipschitz GANs}
 
 \begin{table}
\caption{Inception scores for ordinary vs. proximal training}
\label{table: inception}
\vskip 0.2in
\begin{center}
\begin{small}
\begin{sc}
\begin{tabular}{lcccr}
\toprule
GAN Problem & Ordinary & Proximal  \\
\midrule
WGAN-WC (DIM=64)    & $4.16 \pm 0.15$ & $4.56 \pm 0.19$\\
WGAN-WC (DIM=128) & $2.52 \pm 0.12$ & $4.23 \pm 0.15$ \\
SN-GAN (DIM=64)    & $5.12 \pm 0.25$ & $5.72 \pm 0.22$ \\
SN-GAN (DIM=128)    & $5.62 \pm 0.23$ & $6.12 \pm 0.22$ \\
\bottomrule
\end{tabular}
\end{sc}
\end{small}
\end{center}
\end{table}

We applied the proximal training in Algorithm~\ref{alg:Proximal_Training} to the WGAN-WC and SN-GAN problems. To compute the gradient of the proximal minimax objective, we solved the maximization problem in the Algorithm~\ref{alg:Proximal_Training}'s first step in the for loop by applying $20$ steps of Adam optimization initialized at the discriminator parameters at that iteration. Applying the proximal training to MNIST, CIFAR-10, and CelebA datasets, we qualitatively observed slightly visually better generated pictures. We postpone the generated samples to the Appendix.

To quantitatively compare the proximal and ordinary non-proximal GAN training, we measured the Inception scores of the samples generated in the CIFAR-10 experiments. As shown in Table \ref{table: inception}, proximal training results in an improved inception score. In this table, DIM stands for the dimension parameter of the DC-GAN's CNN networks.

\bibliography{references}
\bibliographystyle{unsrt}
\newpage


\begin{appendices}

\section{Numerical Results for Section 3}
Here, we provide the complete numerical results for the experiments discussed in Section 3 of the main text. Regarding the plots shown in Section 3 for the SN-GAN implementation, here we present the same plots for the Wasserstein GAN with weight clipping (WGAN-WC) and with gradient penalty (WGAN-GP) problems.
Figures \ref{fig:mnist_eq_wc}-\ref{fig:celeba_eq_gp} repeat the experiments of Figures 1,2 in the main text for the WGAN-WC and WGAN-GP problems. These plots suggest that a similar result also holds for the WGAN-WC and WGAN-GP problems, where the objective and the generated samples' quality were decreasing during the generator optimization. For a larger set of generated samples in the main text's Figures 1,2 and Figures \ref{fig:mnist_eq_wc}-\ref{fig:celeba_eq_gp}, we refer the readers to Figures \ref{fig:mnist_sn}-\ref{fig:celeba_gp}.

\begin{figure}
\centering
\begin{subfigure}{0.49\textwidth}
\centering
\includegraphics[width=0.98\linewidth]{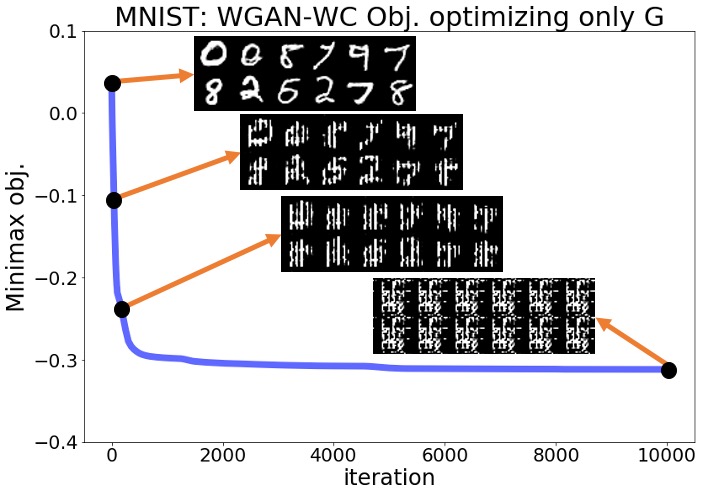}
\caption{Results on the MNIST data} 
\label{fig:mnist_eq_wc}
\end{subfigure}
\begin{subfigure}{.49\textwidth}
\centering
\includegraphics[width=0.98\textwidth]{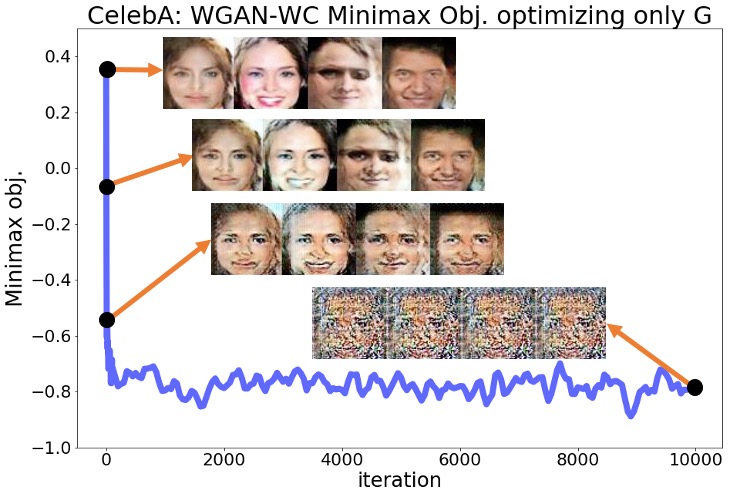}
\caption{Results on the CelebA data} 
\label{fig:celeba_eq_wc}
\end{subfigure}
\caption{Optimizing the trained generator of WGAN-WC with a fixed discriminator on the MNIST and CelebA data. The GAN's objective and samples' quality were decreasing over the optimization.} 
\end{figure}

\begin{figure}
\centering
\begin{subfigure}{0.49\textwidth}
\centering
\includegraphics[width=0.98\textwidth]{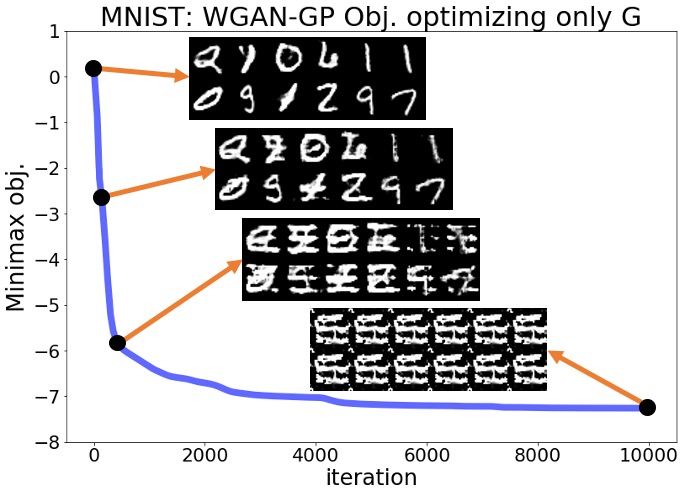}
\caption{Results on the MNIST data} 
\label{fig:mnist_eq_gp}
\end{subfigure}
\begin{subfigure}{.49\textwidth}
\centering
\includegraphics[width=0.98\textwidth]{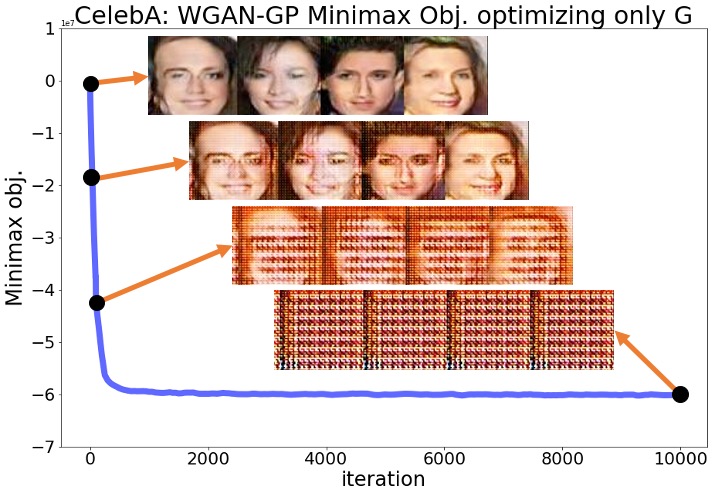}
\caption{Results on the CelebA data} 
\label{fig:celeba_eq_gp}
\end{subfigure}
\caption{Optimizing the trained generator of WGAN-GP with a fixed discriminator on the MNIST and CelebA data. The GAN's objective and samples' quality were decreasing over the optimization.} 
\end{figure}

\begin{figure}
\centering
\begin{subfigure}{0.49\textwidth}
\centering
\includegraphics[height=16cm]{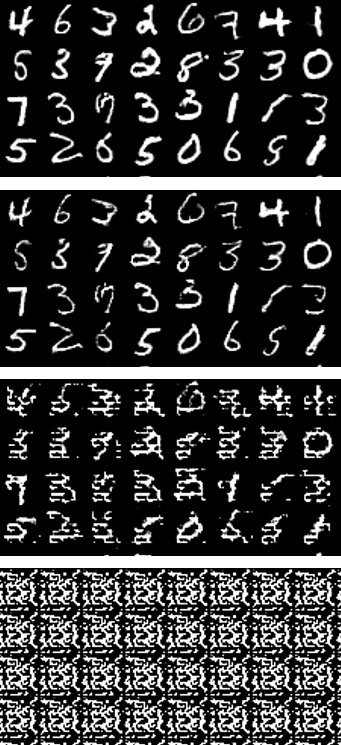}
\caption{Results on the MNIST data} 
\label{fig:mnist_sn}
\end{subfigure}
\begin{subfigure}{.49\textwidth}
\centering
\includegraphics[height=16cm]{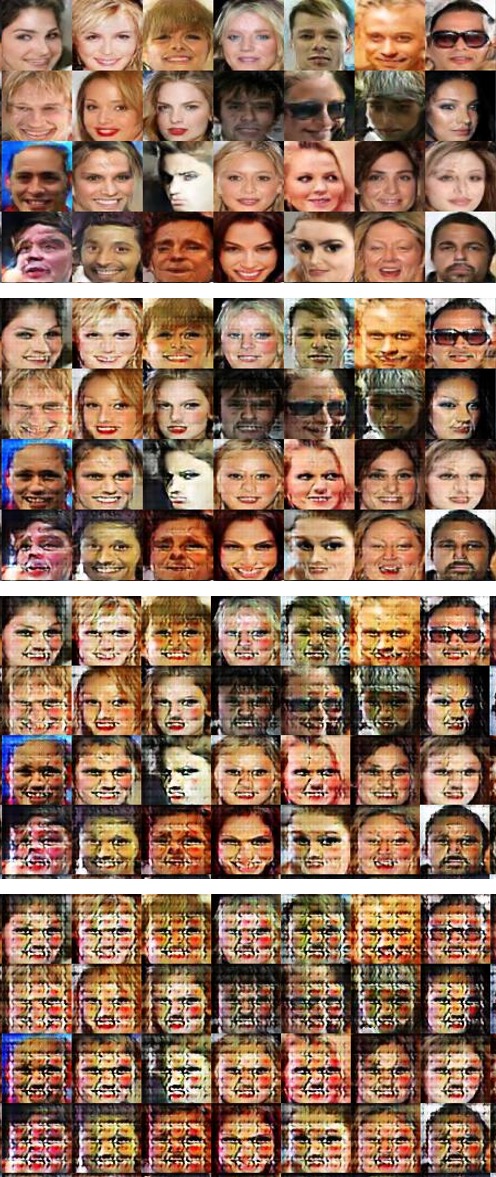}
\caption{Results on the CelebA data} 
\label{fig:celeba_sn}
\end{subfigure}
\caption{SN-GAN's generated samples at the iterations marked in Figures \ref{fig:mnist_eq} \& \ref{fig:celeba_eq}} 
\end{figure}

\begin{figure}
\centering
\begin{subfigure}{0.49\textwidth}
\centering
\includegraphics[height=16cm]{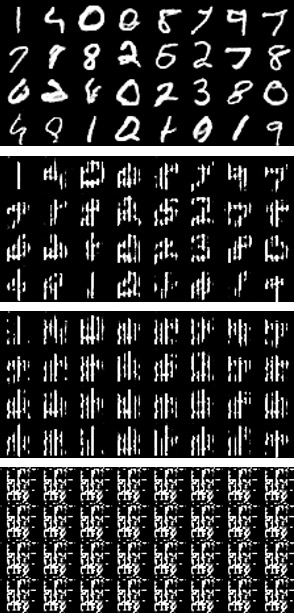}
\caption{Results on the MNIST data} 
\label{fig:mnist_wc}
\end{subfigure}
\begin{subfigure}{.49\textwidth}
\centering
\includegraphics[height=16cm]{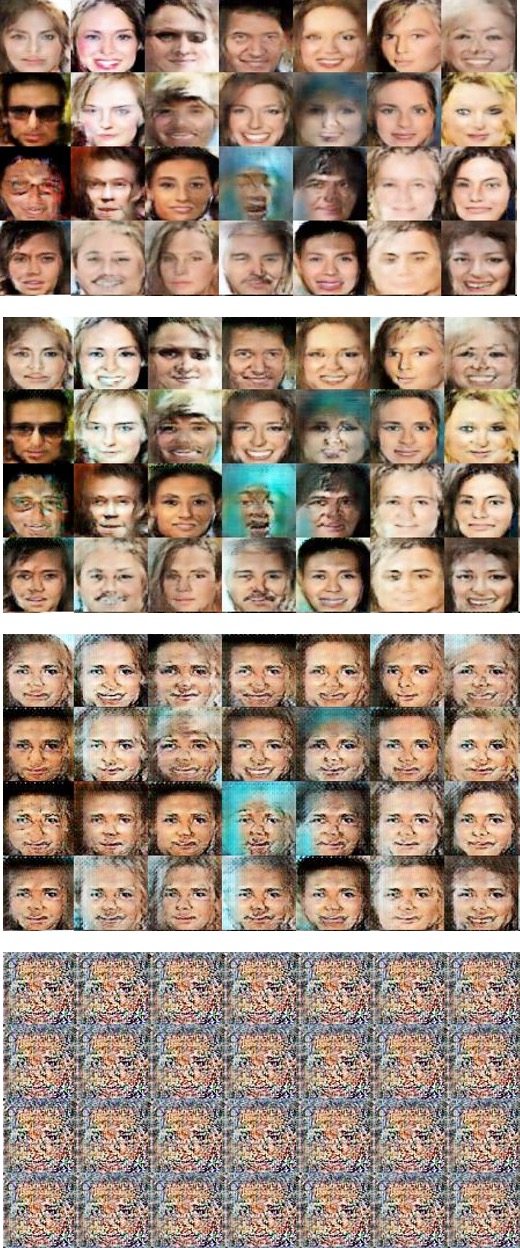}
\caption{Results on the CelebA data} 
\label{fig:celeba_wc}
\end{subfigure}
\caption{WGAN-WC's generated samples at the iterations marked in Figures \ref{fig:mnist_eq_wc} \& \ref{fig:celeba_eq_wc}} 
\end{figure}

\begin{figure}
\centering
\begin{subfigure}{0.49\textwidth}
\centering
\includegraphics[height=16cm]{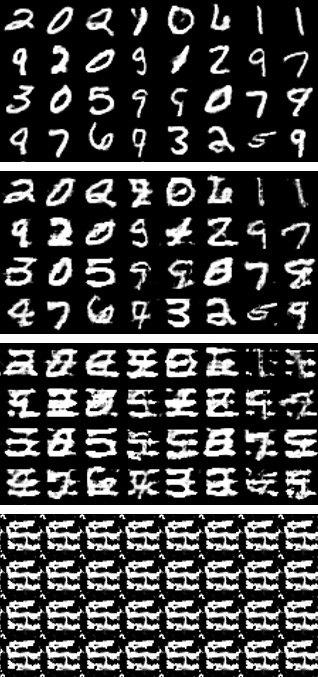}
\caption{Results on the MNIST data} 
\label{fig:mnist_gp}
\end{subfigure}
\begin{subfigure}{.49\textwidth}
\centering
\includegraphics[height=16cm]{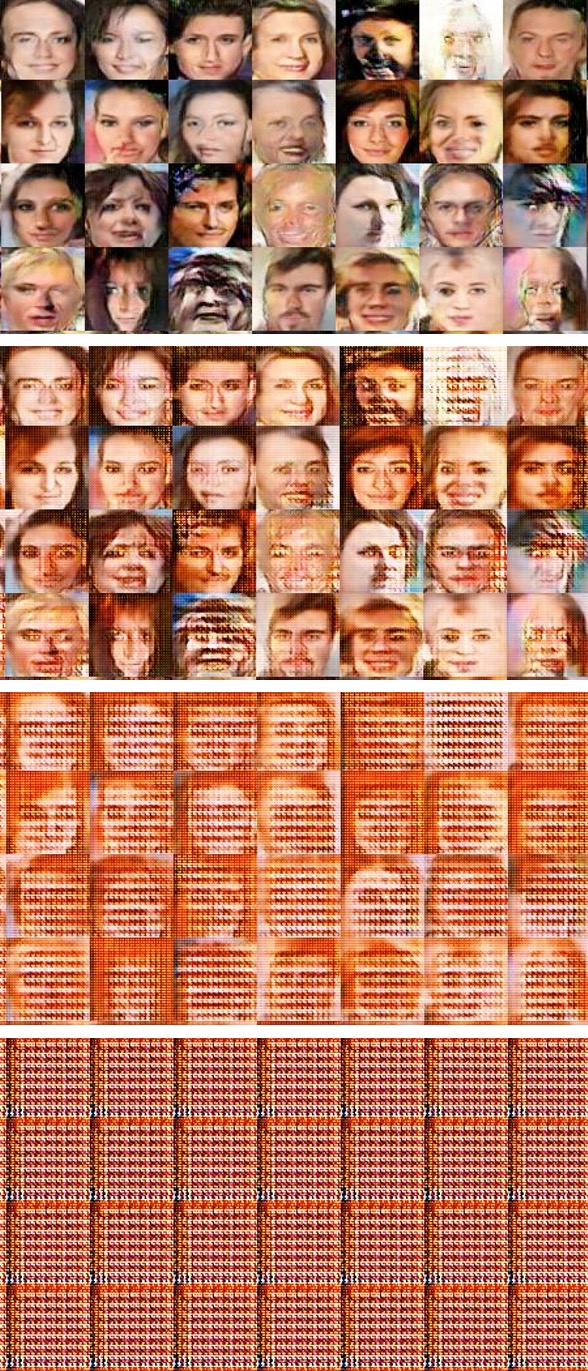}
\caption{Results on the CelebA data} 
\label{fig:celeba_gp}
\end{subfigure}
\caption{WGAN-GP's generated samples at the iterations marked in Figures \ref{fig:mnist_eq_gp} \& \ref{fig:celeba_eq_gp}}
\end{figure}

\section{Numerical Results for Section 9}
Here, we present the complete numerical results for the experiments of Section 9 in the main text. Figures \ref{fig:mnist_eq_pt_wc}-\ref{fig:celeba_eq_pt_gp} demonstrate the results of the main text's Figures 3,4 for the WGAN-WC and WGAN-GP problems. Here, except the WGAN-GP experiment on the CelebA dataset, we observed that the objective and the generated samples' quality did not significantly decrease over the generator optimization. Even for the WGAN-GP experiment on the CelebA data, we observed that the objective value decreased three times less than in minimizing the original objective rather than the proximal objective. These experiments suggest that the Wasserstein and Lipschitz GAN problems can converge to local proximal equilibrium solutions. We also show a larger group of generated samples at the beginning and final iterations of Figures 3,4 in the main text and Figures \ref{fig:mnist_eq_pt_wc}-\ref{fig:celeba_eq_pt_gp} in Figures \ref{fig:mnist_pt_sn}-\ref{fig:celeba_pt_gp}. 

For the proximal training experiments, Figures \ref{fig:pt_cifar_sn}-\ref{fig:pt_celeba_wc} show the samples generated by the SN-GAN and WGAN-WC proximally trained on CIFAR-10 and CelebA data with the results for the baseline regular training on the top of the figure and the results for proximal training on the bottom. We observed a somewhat improved quality achieved by proximal training, which was further supported by the inception scores for the CIFAR-10 experiments reported in the main text.

\begin{figure}[h]
\centering
\begin{subfigure}{0.49\textwidth}
\centering
\includegraphics[width=0.98\textwidth]{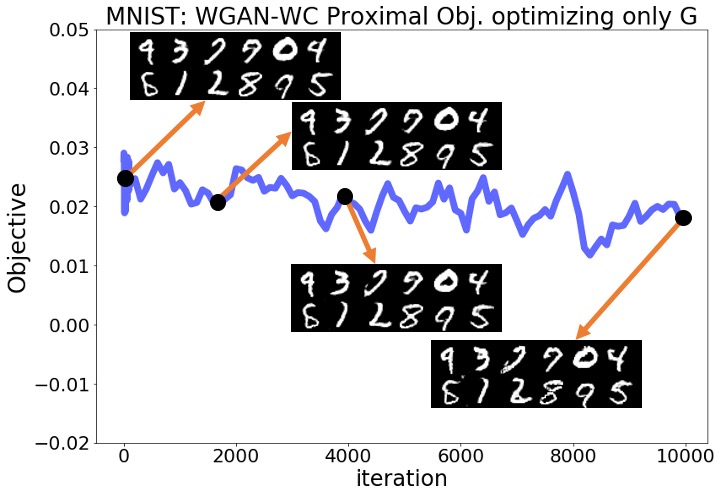}
\caption{Results on the MNIST data} 
\label{fig:mnist_eq_pt_wc}
\end{subfigure}
\begin{subfigure}{.49\textwidth}
\centering
\includegraphics[width=0.98\textwidth]{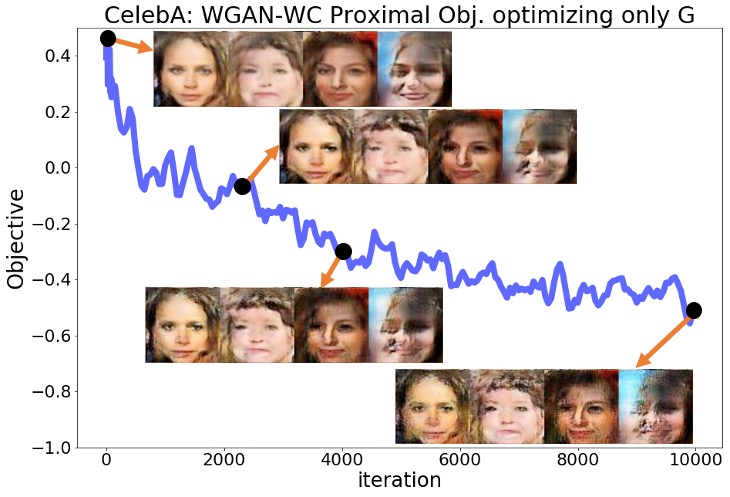}
\caption{Results on the CelebA data} 
\label{fig:celeba_eq_pt_wc}
\end{subfigure}
\caption{Optimizing the proximal objective for the trained generator in WGAN-WC with a fixed discriminator on data. The GAN's objective and samples' quality were preserved over the optimization.}
\end{figure}

\begin{figure}[h]
\centering
\begin{subfigure}{0.49\textwidth}
\centering
\includegraphics[width=0.98\textwidth]{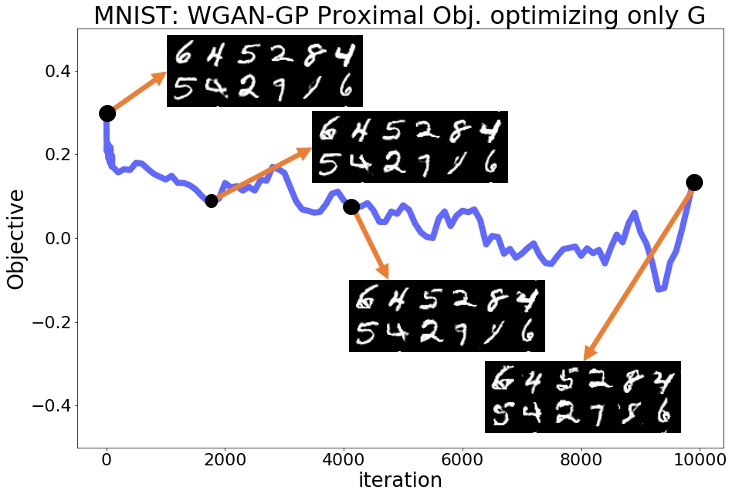}
\caption{Results on the MNIST data} 
\label{fig:mnist_eq_pt_gp}
\end{subfigure}
\begin{subfigure}{.49\textwidth}
\centering
\includegraphics[width=0.98\textwidth]{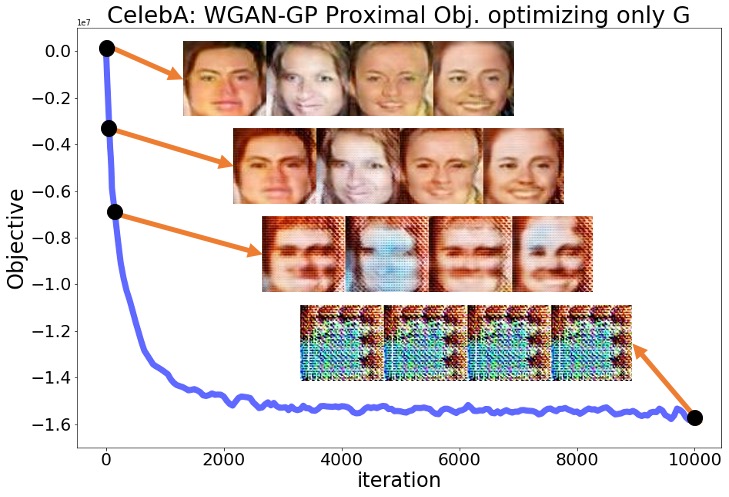}
\caption{Results on the CelebA data} 
\label{fig:celeba_eq_pt_gp}
\end{subfigure}
\caption{Optimizing the proximal objective for the trained generator in WGAN-GP with a fixed discriminator on data. The GAN's objective and samples' quality were preserved over the optimization.}
\end{figure}

\begin{figure}[h]
\centering
\begin{subfigure}{0.49\textwidth}
\centering
\includegraphics[height=9cm]{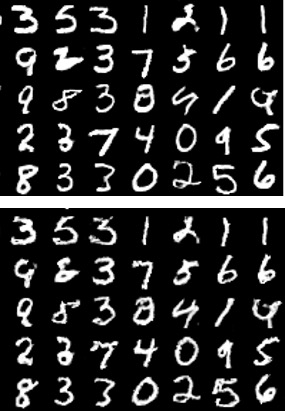}
\caption{Results on the MNIST data} 
\label{fig:mnist_pt_sn}
\end{subfigure}
\begin{subfigure}{.49\textwidth}
\centering
\includegraphics[height=9cm]{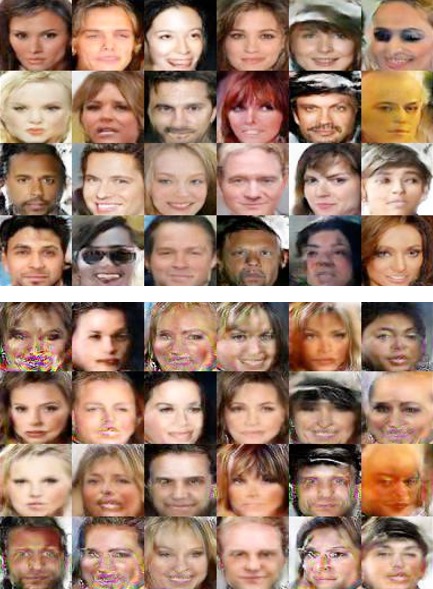}
\caption{Results on the CelebA data} 
\label{fig:celeba_pt_sn}
\end{subfigure}
\caption{SN-GAN's generated samples at the first and last iterations of Figures \ref{fig:mnist_eq_pt},\ref{fig:celeba_eq_pt}} 
\end{figure}

\begin{figure}[h]
\centering
\begin{subfigure}{0.49\textwidth}
\centering
\includegraphics[height=9cm]{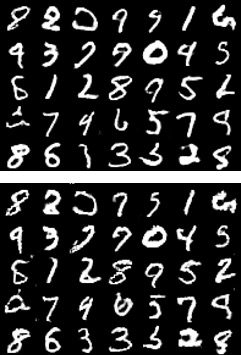}
\caption{Results on the MNIST data} 
\label{fig:mnist_pt_wc}
\end{subfigure}
\begin{subfigure}{.49\textwidth}
\centering
\includegraphics[height=9cm]{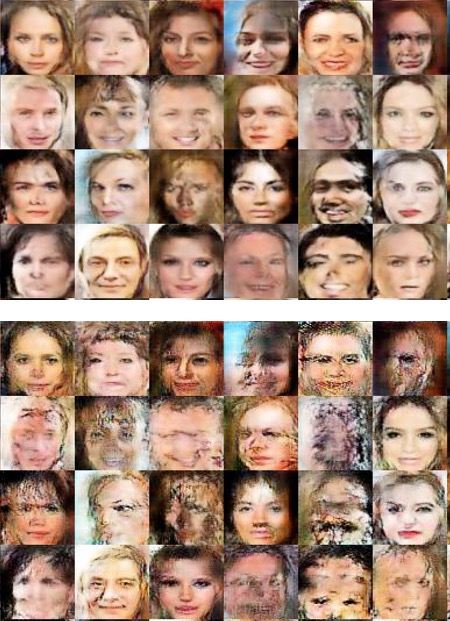}
\caption{Results on the CelebA data} 
\label{fig:celeba_pt_wc}
\end{subfigure}
\caption{WGAN-WC's generated samples at the first and last iterations of Figures \ref{fig:mnist_eq_pt_wc},\ref{fig:celeba_eq_pt_wc}} 
\end{figure}

\begin{figure}[h]
\centering
\begin{subfigure}{0.49\textwidth}
\centering
\includegraphics[height=9cm]{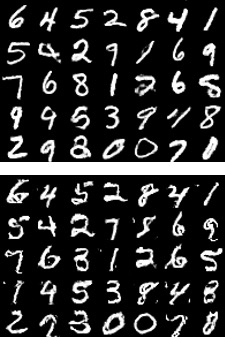}
\caption{Results on the MNIST data} 
\label{fig:mnist_pt_gp}
\end{subfigure}
\begin{subfigure}{.49\textwidth}
\centering
\includegraphics[height=9cm]{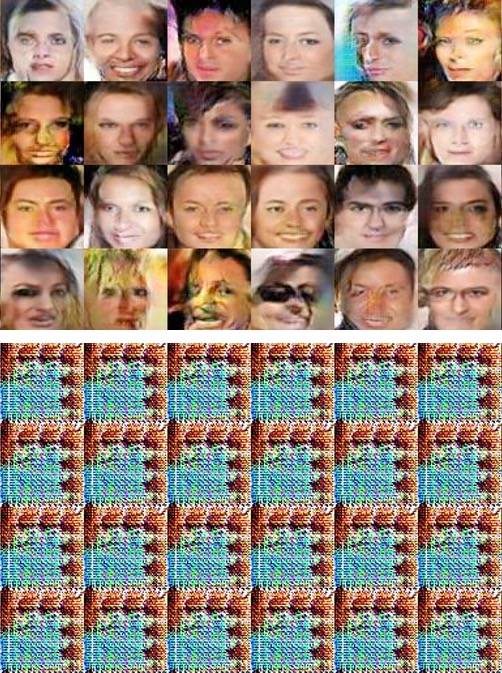}
\caption{Results on the CelebA data} 
\label{fig:celeba_pt_gp}
\end{subfigure}
\caption{WGAN-GP's generated samples at the first and last iterations of Figures \ref{fig:mnist_eq_pt_gp},\ref{fig:celeba_eq_pt_gp}} 
\end{figure}

\begin{figure}[h]
\centering
\includegraphics[width=0.5\textwidth]{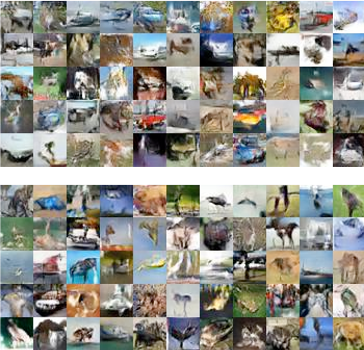}
\caption{The images generated by the SN-GAN (DIM=128) trained on CIFAR-10 data with (top) ordinary and (bottom) proximal training.} 
\label{fig:pt_cifar_sn}
\end{figure}

\begin{figure}[h]
\centering
\begin{subfigure}{0.49\textwidth}
\centering
\includegraphics[width=0.98\textwidth]{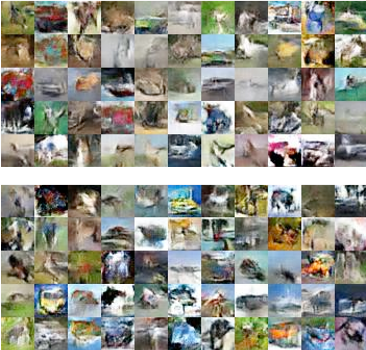}
\caption{WGAN-WC with DCGAN-DIM=64} 
\label{fig:pt_cifar_wc_64}
\end{subfigure}
\begin{subfigure}{.49\textwidth}
\centering
\includegraphics[width=0.98\textwidth]{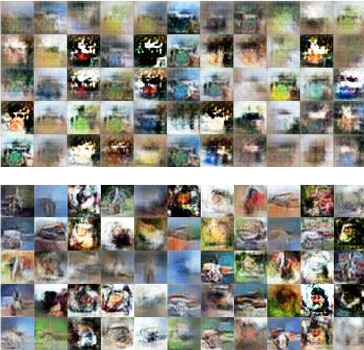}
\caption{WGAN-WC with DCGAN-DIM=128} 
\label{fig:pt_cifar_wc_128}
\end{subfigure}
\caption{The images generated by the WGAN-WC trained on CIFAR-10 data with (top) ordinary and (bottom) proximal training.} 
\end{figure}

\begin{figure}[h]
\centering
\begin{subfigure}{0.49\textwidth}
\centering
\includegraphics[width=0.98\textwidth]{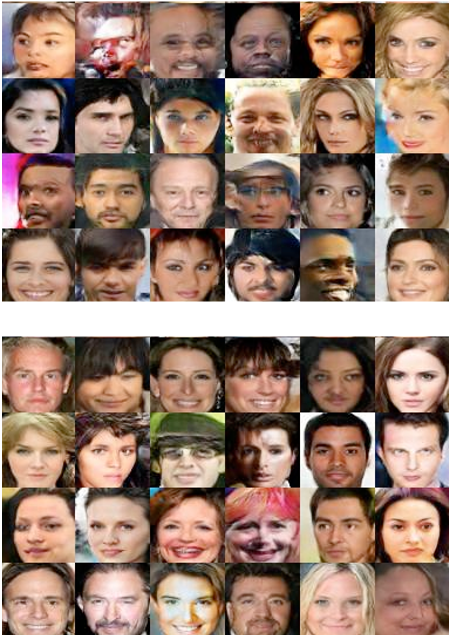}
\caption{The results for SN-GAN} 
\label{fig:pt_celeba_sn}
\end{subfigure}
\begin{subfigure}{.49\textwidth}
\centering
\includegraphics[width=0.99\textwidth]{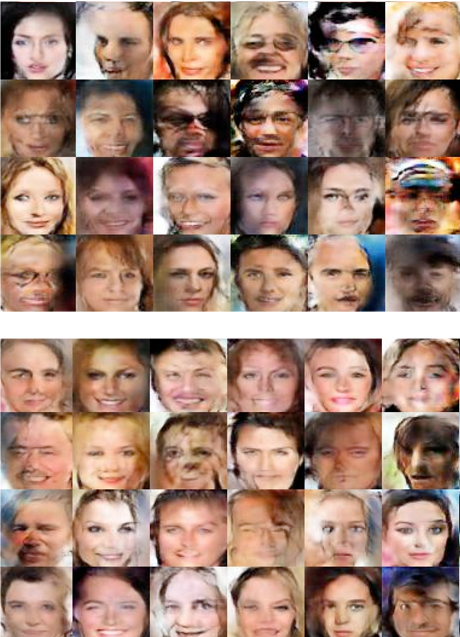}
\caption{The results for WGAN-WC} 
\label{fig:pt_celeba_wc}
\end{subfigure}
\caption{The images generated by the GAN trained on CelebA data with (top) ordinary and (bottom) proximal training.} 
\end{figure}

\section{Proofs}
\subsection{Proof of Proposition 1}

\textbf{Proof for $f$-GANs:} 

Consider the following $f$-GAN minimax problem corresponding to the convex function $f$:
\begin{equation}\label{Prop1 proof: f-GAN}
\min_{G\in \mathcal{G}}\: \max_{D}\: \mathbb{E}\bigl[D(\mathbf{X})\bigr] - \mathbb{E}\bigl[f^*(D(G(\mathbf{Z})))\bigr].   
\end{equation}
Due to the realizability assumption, given $G^*\in\mathcal{G}$ we assume that the data distribution and the generative model are identical, i.e., $P_{\mathbf{X}}=P_{G^*(\mathbf{Z})}$. Then, the minimax objective for $G^*$ reduces to
\begin{equation}
\mathbb{E}_{P_{\mathbf{X}}}\bigl[\, D(\mathbf{X}) -f^*(D(\mathbf{X})) \, \bigr].
\end{equation}
The above objective decouples across $\mathbf{X}$ outcomes. As a result, the maximizing discriminator $D^*(\mathbf{x})=f'(1)$ will be a constant function where the constant value $f'(1)$ follows from the optimization problem:
\begin{equation}
    f'(1) = \underset{u\in\mathbb{R}}{\arg\!\max}\: u -f^*(u). 
\end{equation}
Note that the objective $u -f^*(u)$ is a concave function of $u$ whose derivative is zero at ${{f^*}'}^{-1}(1)=f'(1)$, because the Fenchel-conjugate of a convex $f$ satisfies ${{f^*}'}^{-1} = f'$. 

So far we have proved that the constant function $D_{\operatorname{constant}}(\mathbf{x})=f'(1)$ provides the optimal discriminator for generator $G^*$. Therefore, for every discriminator $D$ we have
\begin{equation}\label{prop1: f-GAN part1}
    V(G^*,D)\le V(G^*,D_{\operatorname{constant}}),
\end{equation}
where $V(G,D)$ denotes the $f$-GAN's minimax objective. Moreover, note that for a constant $D$ the value of the minimax objective does not change with generator $G$. As a result, for every $G$
\begin{equation}\label{prop1: f-GAN part2}
    V(G,D_{\operatorname{constant}}) = V(G^*,D_{\operatorname{constant}}).
\end{equation}
Then, \eqref{prop1: f-GAN part1} and \eqref{prop1: f-GAN part2} collectively prove that for every $G$ and $D$ we have
\begin{equation*}
    V(G^*,D) \le  V(G^*,D_{\operatorname{constant}}) \le V(G,D_{\operatorname{constant}}),
\end{equation*}
which completes the proof for $f$-GANs.

\textbf{Proof for Wasserstein GANs: } 

Consider a general Wasserstein GAN problem with a cost function $c$ satisfying  $c(\mathbf{x},\mathbf{x}) = 0$ for every $\mathbf{x}$. Notice that this property holds for all Wasserstein distance measures corresponding to cost function $\Vert \mathbf{x}-\mathbf{x}' \Vert^q$ for  $q\ge 1$. The generalized Wasserstein GAN minimax problem is as follows:
\begin{equation}
    \min_{G\in\mathcal{G}}\: \max_{D \, \text{\rm $c$-concave}}\: \mathbb{E}[D(\mathbf{X})] - \mathbb{E}\bigl[ D^c(G(\mathbf{Z})) \bigr]. 
\end{equation}
Due to the realizability assumption, a generator function $G^*\in\mathcal{G}$ results in the data distribution such that $P_{G^*(\mathbf{Z})}=P_{\mathbf{X}}$. Then, the above minimax objective for $G^*$ reduces to
\begin{equation}\label{prop1: optimal transport gan, objective}
    \mathbb{E}_{P_{\mathbf{X}}}\bigl[\, D(\mathbf{X}) - D^c(\mathbf{X})\, \bigr]. 
\end{equation}
Since the cost is assumed to take a zero value given identical inputs, we have:
\begin{align*}
    D^c(\mathbf{x}) :&= \max_{\mathbf{x}'}\, D(\mathbf{x}') - c(\mathbf{x},\mathbf{x}') \\
    &\ge D(\mathbf{x}) - c(\mathbf{x},\mathbf{x}) \\
    &= D(\mathbf{x}).
\end{align*}
As a result, $D(\mathbf{x})-D^c(\mathbf{x})\le 0$ holds for every $\mathbf{x}$. Hence, the objective in \eqref{prop1: optimal transport gan, objective} will be non-positive and takes its maximum zero value for any constant function $D_{\operatorname{constant}}$, which by definition satisfies $c$-concavity. Therefore, letting $V(G,D)$ denote the GAN minimax objective, for every $D$ we have
\begin{equation}\label{prop1: OT-GAN part1}
    V(G^*,D) \le V(G^*,D_{\operatorname{constant}}).
\end{equation}
We also know that for a constant discriminator $D_{\operatorname{constant}}$ the value of the minimax objective is independent from the generator function. Therefore, for every $G$ we have
\begin{equation}\label{prop1: OT-GAN part2}
     V(G^*,D_{\operatorname{constant}}) = V(G,D_{\operatorname{constant}}).
\end{equation}
As a result, \eqref{prop1: OT-GAN part1} and \eqref{prop1: OT-GAN part2} together show that for every $G$ and $D$ 
\begin{equation}
      V(G^*,D) \le V(G^*,D_{\operatorname{constant}}) \le V(G,D_{\operatorname{constant}}),
\end{equation}
which makes the proof complete for Wasserstein GANs.

\subsection{Proof of Theorem 1 \& Remark 1}
\textbf{Proof for $f$-GANs:}

\begin{lemma}\label{Lemma f-div}
Consider two random vectors $\mathbf{X}, \widetilde{\mathbf{X}}$ with probability density functions $p,\,q$, respectively. Suppose that $p,q$ are non-zero everywhere. Then, considering the following variational representation of $d_f(P,Q)$,
\begin{equation}
    d_f(P,Q) = \max_{D}\: \mathbb{E}[D(\mathbf{X})] - \mathbb{E}[f^*(D(\widetilde{\mathbf{X}}))],
\end{equation}
the optimal solution $D^*$ will satisfy
\begin{equation}
D^*(\mathbf{x}) = f'\bigl(\frac{p(\mathbf{x})}{q(\mathbf{x})}\bigr). 
\end{equation}
\end{lemma}
\begin{proof}
Let us rewrite the $f$-divergence's variational representation as
\begin{align*}
 d_f(P,Q) \,  &= \, \max_{D}\: \mathbb{E}[D(\mathbf{X})] - \mathbb{E}[f^*(D(\widetilde{\mathbf{X}}))]  \\
 &=\, \max_{D}\: \int \bigl[ p(\mathbf{x})D(\mathbf{x}) - q(\mathbf{x})f^*(D(\mathbf{x}))\bigr]\diff{\mathbf{x}} \\
 & = \int \max_{D(\mathbf{x})}\bigl[ p(\mathbf{x})D(\mathbf{x}) - q(\mathbf{x})f^*(D(\mathbf{x}))\bigr]\diff{\mathbf{x}}
\end{align*}
where the last equality holds, since the maximization objective decouples across $\mathbf{x}$ values. It can be seen that the inside optimization problem for each $D(\mathbf{x})$ is maximizing a concave objective in which by setting the derivative to zero we obtain
\begin{equation}
   {f^*}'(D^*(\mathbf{x})) = \frac{p(\mathbf{x})}{q(\mathbf{x})}.
\end{equation}
As a property of the Fenchel-conjugate of a convex $f$, we know ${{f^*}'}^{-1}=f'$ which combined with the above equation implies that
\begin{equation}
   D^*(\mathbf{x}) = f'\bigr(\frac{p(\mathbf{x})}{q(\mathbf{x})}\bigr).
\end{equation}
The above result completes Lemma \ref{Lemma f-div}'s proof.
\end{proof}
Consider the $f$-GAN problem with the generator function specified in the theorem:
\begin{equation}
    \min_{\mathbf{W},\mathbf{u}:\, \Vert \mathbf{W} \Vert_2\le 1}\: \max_{D}\: \mathbb{E}\bigl[D(\mathbf{X})\bigr] - \mathbb{E}\bigl[f^*\bigl(D(\mathbf{W}\mathbf{Z}+\mathbf{u})\bigr)\bigr].
\end{equation}
Note that $\mathbf{X}\sim \mathcal{N}(\mathbf{0},\sigma^2 I)$ and $\mathbf{W}\mathbf{Z}+\mathbf{u}\sim \mathcal{N}(\mathbf{u},\mathbf{W}\mathbf{W}^T)$. Notice that if $\mathbf{W}$ was not full-rank, the maximized discriminator objective would be $+\infty$ achieved by a $D$ assigning an infinity value to the points not included in the rank-constrained support set of generator $\mathbf{W}\mathbf{z}+\mathbf{u}$. This will not result in a solution to the $f$-GAN problem, because we assume that the dimensions of $\mathbf{X}$ and $\mathbf{Z}$ match each other and hence there exists a full-rank $\mathbf{W}$ with a finite maximized objective, i.e. $f$-divergence value. Therefore, in a Nash equilibrium of the $f$-GAN problem, the solution $\mathbf{W}$ must be full-rank and invertible.

Lemma \ref{Lemma f-div} results in the following equation for the optimal discriminator $D^*_{\mathbf{W},\mathbf{u}}$ given generator parameters $\mathbf{W},\mathbf{u}$:
\begin{align*}
    D^*_{\mathbf{W},\mathbf{u}}(\mathbf{x})\, &=\, f'\bigl( \frac{\frac{1}{\sqrt{(2\pi \sigma^2)^k}}\exp\big\{- \frac{1}{2\sigma^2}\big\Vert \mathbf{x} \big\Vert_2^2\big\} }{\frac{1}{\sqrt{(2\pi)^k \det(\mathbf{W}\mathbf{W}^T)}}\exp\big\{- \frac{1}{2}\big\Vert (\mathbf{W}\mathbf{W}^T)^{-1/2}\mathbf{(x-u)} \big\Vert_2^2\big\}}\bigl) \\
    &= \, f'\bigl(\sqrt{\frac{\det(\mathbf{W}\mathbf{W}^T)}{\sigma^{2k}}} \exp\biggl\{ \frac{1}{2}\mathbf{x}^T((\mathbf{W}\mathbf{W}^T)^{-1} - \sigma^{-2}I)\mathbf{x} - \mathbf{u}^T(\mathbf{W}\mathbf{W}^T)^{-1/2}\mathbf{x} +\mathbf{u}^T (\mathbf{W}\mathbf{W}^T)^{-1}\mathbf{u} \biggr\}\, \bigr).
\end{align*}
As a result, the function $f^*(D^*_{\mathbf{W},\mathbf{u}}(\cdot))$ appearing in the $f$-GAN's minimax objective will be
\begin{align*}
    {f^*}\bigl(D^*_{\mathbf{W},\mathbf{u}}(\mathbf{x})\bigr) \, =\, & {f^*}\biggl(f'\biggl( \sqrt{\frac{\det(\mathbf{W}\mathbf{W}^T)}{\sigma^{2k}}} \exp\biggl\{ \frac{1}{2}\mathbf{x}^T( (\mathbf{W}^T\mathbf{W})^{-1} \\
     - & \sigma^{-2}I)\mathbf{x} - \mathbf{u}^T(\mathbf{W}\mathbf{W}^T)^{-1/2}\mathbf{x} +\mathbf{u}^T (\mathbf{W}\mathbf{W}^T)^{-1}\mathbf{u} \biggr\}\biggr)\biggr).
\end{align*}
\textbf{Claim:} $f^*(D^*_{\mathbf{W},\mathbf{u}}(\mathbf{x}))$ is a strictly convex function of $\mathbf{x}$.

To show this claim, note that the following expression is a strongly-convex quadratic function of $\mathbf{x}$, since we have assumed that the spectral norm of $\mathbf{W}$ is bounded as ${\sigma}_{\max}( \mathbf{W})\le 1 < \sigma$:
\begin{align*}
\frac{1}{2}\mathbf{x}^T( (\mathbf{W}^T\mathbf{W})^{-1} - \sigma^{-2}I)\mathbf{x} - \mathbf{u}^T(\mathbf{W}^T\mathbf{W})^{-1/2}\mathbf{x} \mathbf{u}^T (\mathbf{W}^T\mathbf{W})^{-1}\mathbf{u}.
\end{align*}
For simplicity, we denote the above strongly-convex function with $g(\mathbf{x})$ and define the function $h:\mathbb{R}\rightarrow\mathbb{R}$ as 
\begin{align*}
    h(y) := {f^*}\bigl({{f}'}\bigl( \sqrt{\frac{\det(\mathbf{W}\mathbf{W}^T)}{\sigma^{2k}}}\times  e^y \bigr) \bigr).
\end{align*}
According to the above definitions, $f^*(D^*_{\mathbf{W},\mathbf{u}}(\mathbf{x})) =h(g(\mathbf{x}))$ is the composition of $h$ and strongly-convex $g$. Note that $h$ is a monotonically increasing function, since defining $c=\sqrt{\frac{\det(\mathbf{W}\mathbf{W}^T)}{\sigma^{2k}}} > 0$ we have
\begin{equation}
    h'(y) = (c\, e^{y})^2\, f''(c\,e^y)\ge 0, 
\end{equation}
which follows from the equality $${f^*}(f'(z)) :=\sup_u \{ uf'(z)-f(u) \} \, = \, zf'(z) - f(z)$$ that is a consequence of the definition of Fenchel-conjugate, implying that
$\frac{\diff{}{f^*}(f'(z))}{\diff{z}} = zf''(z)$ for the convex $f$. Note that $h'(y)>0$ holds everywhere, because $f$ is assumed to be strictly convex. This proves that $h$ is strictly increasing. 
Furthermore, $h$ is a convex function, because $h'(y)$ is non-decreasing due to the assumption that $t^2 f''(t)$ is non-decreasing over $t\in (0,+\infty)$. As a result, $h$ is an increasing convex function.

Therefore, $f^*(D^*_{\mathbf{W},\mathbf{u}}(\mathbf{x})) =h(g(\mathbf{x}))$ is a composition of a strongly-convex $g$ and an increasing convex $h:\mathbb{R}\rightarrow\mathbb{R}$. Therefore, as a well-known result in convex optimization \cite{boyd2004convex}, the claim is true and $f^*(D^*_{\mathbf{W},\mathbf{u}}(\mathbf{x}))$ is a strictly convex function of $\mathbf{x}$.

We showed that the claim is true for every feasible $\mathbf{W}, \mathbf{u}$. Now, we prove that the pair $(G_{\mathbf{W},\mathbf{u}},D^*_{\mathbf{W},\mathbf{u}})$ will not be a local Nash equilibrium for any feasible $\mathbf{W}, \mathbf{u}$. If the pair $(G_{\mathbf{W},\mathbf{u}},D^*_{\mathbf{W},\mathbf{u}})$ was a local Nash equilibrium, $\mathbf{W},\mathbf{u}$ would be a local minimum for the following minimax objective where $D^*$ is fixed to be $D^*_{\mathbf{W},\mathbf{u}}$:
\begin{equation}\label{thm: f-GAN no Nash}
    \mathbb{E}\bigl[D^*(\mathbf{X})\bigr] - \mathbb{E}\bigl[f^*\bigl(D^*(\mathbf{W}\mathbf{Z}+\mathbf{u})\bigr)\bigr].
\end{equation}
However, as shown earlier, for any feasible $\mathbf{W},\mathbf{u}$, $f^*\bigl(D^*(\mathbf{x}))$ is a strictly-convex function of $\mathbf{x}$, which in turn shows that \eqref{thm: f-GAN no Nash} is a strictly-concave function of variable $\mathbf{u}$. This consequence proves that the objective has no local minima for the unconstrained variable $\mathbf{u}$. Due to the shown contradiction, a pair with the form $(G_{\mathbf{W},\mathbf{u}},D^*_{\mathbf{W},\mathbf{u}})$ cannot be a local Nash equilibrium in parameters $\mathbf{W},\mathbf{u}$. Consequently, the minimax problem has no pure Nash equilibrium solutions, since in a pure Nash equilibrium the discriminator will be by definition optimal against the choice of generator.

\textbf{Proof for W2GANs:} 

Consider the W2GAN problem with the assumed generator function: 
\begin{equation}
    \min_{\mathbf{W},\mathbf{u}:\, \Vert \mathbf{W} \Vert_2\le 1}\: \max_{D\,\text{\rm $c$-concave} }\: \mathbb{E}\bigl[D(\mathbf{X})\bigr] - \mathbb{E}\bigl[D^c(\mathbf{W}\mathbf{Z}+\mathbf{u})\bigr],
\end{equation}
where the c-transform is defined for the quadratic cost function $c(\mathbf{x},\mathbf{x}')=\frac{1}{2}\Vert \mathbf{x}-\mathbf{x}'\Vert_2^2$. Similar to the $f$-GAN case, define $D^*_{\mathbf{W},\mathbf{u}}$ to be the optimal discriminator for the generator function parameterized by $\mathbf{W},\mathbf{u}$. Note that $\mathbf{X}\sim \mathcal{N}(\mathbf{0},\sigma^2 I)$ and $\mathbf{W}\mathbf{Z}+\mathbf{u}\sim \mathcal{N}(\mathbf{u},\mathbf{W}\mathbf{W}^T)$. 

According to the Brenier's theorem \cite{villani2008optimal}, the optimal transport from the Gaussian data distribution to the Gaussian generative model will be $$\psi^{\operatorname{opt}}(\mathbf{x}) = \mathbf{x} - \nabla_{\mathbf{x}} D^{*}_{\mathbf{W},\mathbf{u}} (\mathbf{x}).\footnote{Notice the change of variable $D(\mathbf{x})=\frac{1}{2}\Vert\mathbf{x} \Vert^2-\psi(\mathbf{x})$ compared to the formulation discussed at \cite{villani2008optimal,feizi2017understanding} which are based on the function $\psi$.} $$ 

As a well-known result regarding the second-order optimal transport map between two Gaussian distributions, the optimal transport will be a linear transformation as $\psi^{\operatorname{opt}}(\mathbf{x}) = \frac{1}{\sigma}(\mathbf{W}\mathbf{W}^T)^{1/2}\mathbf{x}+\mathbf{u}$. This result shows that
\begin{equation}
 \nabla_{\mathbf{x}} D^{*}_{\mathbf{W},\mathbf{u}} (\mathbf{x}) =  \bigl(I-\frac{1}{\sigma}(\mathbf{W}\mathbf{W}^T)^{1/2} \bigr)\mathbf{x}-\mathbf{u}.  
\end{equation}
Note that the c-transform for cost $c(\mathbf{x},\mathbf{x}')=\frac{1}{2}\Vert \mathbf{x} - \mathbf{x}'\Vert^2_2$ satisfies $D^{c}(\mathbf{x}) = (\frac{1}{2}\Vert\mathbf{x}\Vert^2-D(\mathbf{x}))^* - \frac{1}{2}\Vert\mathbf{x}\Vert^2$ where $g^*$ denotes $g$'s Fenchel-conjugate. For general convex quadratic function $g(\mathbf{x})=\frac{1}{2}\mathbf{x}^TA\mathbf{x} + \mathbf{b}^T\mathbf{x}$ we have $g^*(\mathbf{x})=\frac{1}{2}(\mathbf{x}-\mathbf{b})^TA^{\dagger}(\mathbf{x}-\mathbf{b})$ where $A^{\dagger}$ denotes $A$'s Moore Penrose pseudoinverse. Therefore, for the c-transform of the optimal discriminator we will have
\begin{align*}
 \nabla_{\mathbf{x}} D^{*^c}_{\mathbf{W},\mathbf{u}} (\mathbf{x}) = \,  \bigl(\sigma((\mathbf{W}\mathbf{W}^T)^{1/2})^{\dagger} - I\bigr)\mathbf{x} - \sigma((\mathbf{W}\mathbf{W}^T)^{1/2})^{\dagger}\mathbf{u}.  
\end{align*}
Since every feasible $\mathbf{W}$ satisfies the bounded spectral norm condition as ${\sigma}_{\max} (\mathbf{W})\le 1 < \sigma$, the optimal $D^{*^c}_{\mathbf{W},\mathbf{u}}$ will be a quadratic function whose Hessian has at least one strictly positive eigenvalue along the principal eigenvector of $\mathbf{W}\mathbf{W}^T$. The positive eigenvalue exists in general case where $\mathbf{Z}$'s dimension can be even smaller than $\mathbf{X}$'s dimension. If we had the stronger assumption that the two dimensions exactly match, similar to the f-GAN problem considered, then the pseudo-inverse $A^{\dagger}$ would be the same as the inverse $A^{-1}$ resulting in a strongly-convex quadratic $D^{*^c}_{\mathbf{W},\mathbf{u}}$. Nevertheless, as we prove here, the theorem's result on W2GAN holds in the general case and does not necessarily require the same dimension between $\mathbf{X}$ and $\mathbf{Z}$. 

Consider the W2GAN minimax objective for the pair $(G_{\mathbf{W,u}},D^*)$ where $D^*$ is fixed to be the optimal $D^*_{\mathbf{W,u}}$:
\begin{equation}
  \mathbb{E}\bigl[D(\mathbf{X})\bigr] - \mathbb{E}\bigl[{D^*}^c(\mathbf{W}\mathbf{Z}+\mathbf{u})\bigr],
\end{equation}
If $(G_{\mathbf{W,u}},D^*_{\mathbf{W,u}})$ was a local Nash equilibrium, the variables $\mathbf{W,u}$ would provide a local minimum to the above objective. However, since $D^{*^c}_{\mathbf{W},\mathbf{u}}$ is shown to be a quadratic function with a Hessian possessing positive eigenvalues, the above minimax objective will not have a local minimum in the unconstrained variable $\mathbf{u}$. Therefore, the minimax problem possesses no local Nash equilibrium solutions with the form $(G_{\mathbf{W,u}},D^*_{\mathbf{W,u}})$ and therefore no pure Nash equilibrium solutions.

For the parameterized case with a quadratic discriminator $D_{A,\mathbf{b}}(\mathbf{x})=\mathbf{x}^TA\mathbf{x}+\mathbf{b}^T\mathbf{x}$, first of all note that as shown in the proof the optimal discriminator $D^*_{\mathbf{W,u}}$ for any generator parameter $\mathbf{W,u}$ will be a $c$-concave quadratic function. Therefore, the optimal solution for the discriminator does not change because of the new quadratic constraint.  Furthermore, the discriminator optimization problem has a concave objective in parameters $A,\mathbf{b}$. This is because the discriminator $D_{A,\mathbf{b}}(\mathbf{x})$ is a linear function in terms of $A,\mathbf{b}$, and $D^c_{A,\mathbf{b}}(\mathbf{x}) = \sup_{\mathbf{x}'} D_{A,\mathbf{b}}(\mathbf{x}') - c(\mathbf{x}',\mathbf{x})$ is a convex function of $A,\mathbf{b}$ as the supremum of some affine functions is convex. 

As a result, the discriminator optimization reduces to maximizing a concave objective of $A,\mathbf{b}$ constrained to a convex set $\{A:\, I-A\succcurlyeq 0\}$ which is equivalent to the $c$-concave constraint on the quadratic $D_{A,\mathbf{b}}$. Hence, any local solution to this optimization problem will also be a global solution. This result implies that any local Nash equilibrium for the new parameterized minimax problem will have the form $(G_{\mathbf{W,u}},D^*_{\mathbf{W,u}})$, which as we have already shown does not exist under the theorem's assumptions.  

\textbf{Proof for the 1-dimensional WGAN:} 

Consider the 1-dimensional Wasserstein GAN problem for the assumed linear generator function:
\begin{equation}
    \min_{w,u:\, \vert w \vert\le 1}\: \max_{\Vert D\Vert_{\operatorname{Lip}}\le 1}\: \mathbb{E}[D(X)]-\mathbb{E}[D(wZ+u)].
\end{equation}
The inner maximization problem can be rewritten as
\begin{equation}\label{thm1: w1 proof}
    \max_{\Vert D\Vert_{\operatorname{Lip}}\le 1}\: \int \bigl( p_X(x) - p_{wZ+u}(x) \bigr)\, D(x)\diff{x}. 
\end{equation}
Here we have
\begin{align*}
p_X(x) - p_{wZ+u}(x)\, = \, \frac{1}{\sqrt{2\pi\sigma^2}}\exp\bigl\{\frac{-1}{2\sigma^2} x^2 \bigr\} - \frac{1}{\sqrt{2\pi w^2}}\exp\bigl\{\frac{-1}{2 w^2} (x-u)^2 \bigr\}
\end{align*}
Since $|w|\le 1<\sigma$, it can be seen that the above difference will be positive everywhere except over an interval $(a_1,a_2)$, where $a_1,a_2$ are the two solutions to the quadratic equation: 
\begin{equation}
   \bigl(\frac{1}{w^2}-\frac{1}{{\sigma}^2}\bigr)x^2 - 2\frac{u}{w^2}x + \bigl(\frac{u^2}{w^2}-\log(\frac{\sigma}{\vert w\vert})\bigr) = 0. 
\end{equation}
Note that the above quadratic equation has two distinct solutions $a_1 < a_2$, since $|w|<\sigma$ and $\log(\frac{\sigma}{\vert w\vert}) > 0$ leading to the positive discriminant:
\begin{equation}
   4\frac{u^2}{w^2{\sigma}^2} + 4\log(\frac{\sigma}{\vert w\vert}) \bigl(\frac{1}{w^2}-\frac{1}{{\sigma}^2}\bigr)\, >\, 0.
\end{equation}
As the function $D$ in the maximization problem \eqref{thm1: w1 proof} is only constrained to be 1-Lipschitz, the optimal $D^*_{w,u}$'s slope must be equal to $-1$ over $(-\infty,a_1]$ and equal to $+1$ over $[a_2,\infty)$, in order to allow the maximum increase in the maximization objective. Over the interval $(a_1,a_2)$, we claim that for the optimal $D$ is a convex function, because otherwise its double Fenchel-conjugate $D^{**}$, which is by definition convex, achieves a higher value. 

First of all, note that the double Fenchel-conjugate $D^{**}$ will not be different from $D$ outside the $(a_1,a_2)$ interval, because $D^{**}$ is defined to provide the largest convex function satisfying $D^{**}\le D$, and $D$ is supposed to be $1$-Lipschitz taking its minimum derivative on $(\infty,a_1]$ and its maximum derivative over $[a_2,\infty)$. Next, since $D^{**}$ lower-bounds $D$, it results in a non-smaller integral value over the interval $(a_1,a_2)$ as $p_X(x) - p_{wZ+u}(x)$ takes negative values over $(a_1,a_2)$. If $D$ is not convex, then $D^{**}$ provides a strict lower-bound for $D$ which matches $D$ over $(\infty,a_1]\cup [a_2,\infty)$. Therefore, the convex $1$-Lipschitz $D^{**}$ results in a greater objective that is a contradiction to $D$'s optimality. This contradiction proves that the optimal discriminator $D^*_{w,u}$ is a convex function.

Therefore, for every feasible $|w|\le 1$, there exists an optimal solution $D^*_{w,u}$ for \eqref{thm1: w1 proof} that is a non-constant convex function. This result proves that the WGAN problem has no local Nash equilibiria with the form $(G_{w,u},D^*_{w,u})$, because if $(G_{w,u},D^*_{w,u})$ was a local Nash equilibrium then $w,u$ would be a local minimum for the following objective where $D^*$ is fixed to be $D^*_{w,u}$:
\begin{equation}
    \mathbb{E}[D^*(X)] - \mathbb{E}[D^*(wZ+u)].
\end{equation}
However, the above objective is a non-constant concave function of the unconstrained variable $u$ and hence does not have a local minimum in $u$. This shows that the WGAN problem does not have a Nash equilibrium and completes the proof for the WGAN case. 

\begin{remark*}
Consider the same setting as in Theorem \ref{Thm: no Nash}. However, unlike Theorem \ref{Thm: no Nash} suppose that $\sigma < 1$ and $\sigma_{\min}( \mathbf{W} ) \ge 1$ where $\sigma_{\min}$ stands for the minimum singular value. Then, for the WGAN and W2GAN problems described in Theorem \ref{Thm: no Nash}, the Wasserstein distance-minimizing generator results in a Nash equilibrium.
\end{remark*}
\begin{proof}
\textbf{Proof for the W2GAN:}

For the W2GAN case, note that if we repeat the same steps as in the proof of Theorem 1, we can show 
\begin{align}\label{remark1, w2}
 \nabla_{\mathbf{x}} D^{*^c}_{\mathbf{W},\mathbf{u}} (\mathbf{x}) =\, (\sigma(\mathbf{W}\mathbf{W}^T)^{-1/2} - I\bigr)\mathbf{x} -\sigma(\mathbf{W}\mathbf{W}^T)^{-1/2}\mathbf{u}.  \nonumber
\end{align}
which is a concave quadratic function of $\mathbf{x}$, since the assumptions imply that $(\mathbf{W}\mathbf{W}^T)^{-1} \preccurlyeq \sigma^{-2} I$. Here $\mathbf{W}$ is supposed to be a full-rank square matrix as its minimum singular value is assumed to be positive and $\mathbf{Z}$ has the same dimension as $\mathbf{X}$. 

We claim that for the feasible choice $\mathbf{W}^*=I$ and $\mathbf{u}^*=\mathbf{0}$, the pair $(G_{\mathbf{W}^*,\mathbf{u}^*},D^*_{\mathbf{W}^*,\mathbf{u}^*})$ results in a Nash equilibrium of the minimax problem. Considering the definition of the optimal discriminator $D^*_{\mathbf{W}^*,\mathbf{u}^*}$, its optimlaity for $G_{\mathbf{W}^*,\mathbf{u}^*}$ directly follows. Moreover, \eqref{remark1, w2} implies that
\begin{equation}
    D^{*^c}_{\mathbf{W}^*,\mathbf{u}^*} (\mathbf{x}) =\frac{\sigma - 1}{2}\Vert\mathbf{x}\Vert_2^2.  
\end{equation}
As a result, fixing the above discriminator function the minimax objective will be
\begin{align*}
    &\mathbb{E}[D^*(\mathbf{X})] - \mathbb{E}\bigl[\, \frac{\sigma - 1}{2}\, \Vert\mathbf{W}\mathbf{Z}+\mathbf{u}\Vert_2^2\bigr] \\
    =\; & \mathbb{E}[D^*(\mathbf{X})] + \frac{1-\sigma}{2}\,\bigl( \Vert \mathbf{W}\Vert^2_F + \Vert\mathbf{u}\Vert^2_2\bigr)
\end{align*}
which is minimized at $\mathbf{W}=I$ and $\mathbf{u}=\mathbf{0}$ over the specified feasible set, as we know the Frobenius norm-squared, $\Vert \mathbf{W}\Vert^2_F$, is the sum of the squared of $\mathbf{W}$'s singular values. Therefore, the claim holds and the choice $\mathbf{W}^*=I$ and $\mathbf{u}^*=0$ results in the optimal solution and a Nash equilibrium.

\textbf{Proof for the 1-dimensional Wasserstein GAN:}

Here we select the parameters $w^*=1$, $u^*=0$. We claim that the optimal discriminator function for this choice is the negative absolute value function $D^*(x)=-|x|$. Note that the optimal 1-Lipschitz $D^*$ solves the following problem:
\begin{equation}
    \max_{\Vert D \Vert_{\operatorname{Lip}}\le 1}\: \int \bigl( \frac{1}{\sqrt{2\pi\sigma^2}}e^{-\frac{x^2}{2\sigma^2}} - \frac{1
    }{\sqrt{2\pi}}e^{-\frac{x^2}{2}}\bigr) D(x)\diff{x}. 
\end{equation}
In the above objective given $\eta = \sqrt{\frac{2\sigma^2\log(1/\sigma)}{1-\sigma^2}}$, the function $\frac{1}{\sqrt{2\pi\sigma^2}}e^{-\frac{x^2}{2\sigma^2}} - \frac{1
    }{\sqrt{2\pi}}e^{-\frac{x^2}{2}}$ is positive over $(-\eta,\eta)$ and negative elsewhere. Therefore, the optimal $D^*$ should get the maximum $+1$ derivative over $(-\infty,-\eta]$ and the minimum $-1$ derivative over $[+\eta,+\infty)$. Because of the even structure of $\frac{1}{\sqrt{2\pi\sigma^2}}e^{-\frac{x^2}{2\sigma^2}} - \frac{1
    }{\sqrt{2\pi}}e^{-\frac{x^2}{2}}$, there exists an even optimal $D^*$ because $\frac{D^*(x)+D^*(-x)}{2}$ remains $1$-Lipschitz and optimal for any optimal $1$-Lipschitz discriminator $D^*$. The optimal even $D^*$ should further be continuous as a $1$-Lipschitz function, implying that such a $D^*$ is decreasing over $(0,\eta]$ and increasing over $[-\eta,0)$. Enforcing the maximum derivative over the two interval results in the optimal $D^*(x)=-|x|$. 
    
    Therefore, $D^*(x)=-|x|$ provides an optimal  discriminator for $w^*=1,\,u^*=0$. Also, for this $D^*$ the minimax objective of the Wasserstein GAN will be
    \begin{align*}
        \mathbb{E}[-|X|] - \mathbb{E}[-|wZ+u|] \,
        \,= - \mathbb{E}[|X|] + \mathbb{E}[|wZ+u|]
    \end{align*}
    In the above equation, $wZ+u\sim \mathcal{N}(u,w^2)$, showing that the above objective is minimized at $w^*=1, u =0 $ considering the assumed feasible set where $|w|\ge 1$. As a result, the pair $(G_{w^*,u^*}, D^*)$ provides a Nash equilibrium to the WGAN minimax game.
\end{proof}

\subsection{Proof of Proposition 2}

\textbf{Proof of the $\Rightarrow$ direction:} 

Assume that $(G^*,D^*)$ is a $\lambda$-proximal equilibrium.
According to the definition of the proximal equilibrium, the following holds for every $G\in\mathcal{G}$ and $D\in\mathcal{D}$:
\begin{equation}\label{prop2: prox equilibirum holds}
V^{\operatorname{prox}}_{\lambda}(G^*,D)\, \le \, V^{\operatorname{prox}}_{\lambda}(G^*,D^*) \, \le \, V^{\operatorname{prox}}_{\lambda}(G,D^*).  
\end{equation}
\textbf{Claim:} $V^{\operatorname{prox}}_{\lambda}(G^*,D^*) = V(G^*,D^*)$.

To show this claim, note that
\begin{equation}
    V^{\operatorname{prox}}_{\lambda}(G^*,D^*) := \max_{\widetilde{D}\in\mathcal{D}}\: V(G^*,\widetilde{D}) - \frac{\lambda}{2}\bigl\Vert \widetilde{D} - D^* \bigr\Vert^2. 
\end{equation}
In this optimization, the optimal solution $\tilde{D}$ is $D^*$ itself. Otherwise, for the optimal $\tilde{D}\in\mathcal{D}$ we have $\Vert \widetilde{D} - D^* \bigr\Vert > 0$ and as a result
\begin{equation}
    V^{\operatorname{prox}}_{\lambda}(G^*,D^*) < V(G^*,\tilde{D}) \le V^{\operatorname{prox}}_{\lambda}(G^*,\tilde{D}),
\end{equation}
which is a contradiction given that $(G^*,D^*)$ is a $\lambda$-proximal equilibrium. Therefore, $D^*$ optimizes the proximal optimization, which shows the claim is valid and we have $V^{\operatorname{prox}}_{\lambda}(G^*,D^*) = VG^*,D^*)$. Knowing that $V(G,D)\le V^{\operatorname{prox}}_{\lambda}(G,D)$ holds for every $G\in\mathcal{G}, D\in\mathcal{D}$, 
we have
\begin{align*}
   V(G^*,D)&\le V^{\operatorname{prox}}_{\lambda}(G^*,D)\\
   & \le V^{\operatorname{prox}}_{\lambda}(G^*,D^*)\\ &= V(G^*,D^*).
\end{align*}
Furthermore,
\begin{align}\label{prop2: combined equations}
 V(G^*,D^*) &= \, V^{\operatorname{prox}}_{\lambda}(G^*,D^*) \\
&\le \, V^{\operatorname{prox}}_{\lambda}(G,D^*) \\
&= \, \max_{\widetilde{D}\in\mathcal{D}}\: V(G,\widetilde{D}) - \frac{\lambda}{2}\bigl\Vert \widetilde{D} - D^* \bigr\Vert^2. \nonumber
\end{align}
Therefore, the proof is complete.

\textbf{Proof of the $\Leftarrow$ direction:} 

Suppose that for $(G^*,D^*)$ the following holds for every $G\in\mathcal{G}$ and $D\in\mathcal{D}$:
\begin{align}\label{prop2:proof the reverse direction}
        V(G^*,D)\, \le \, V(G^*,D^*)  \,
         \le \, \max_{\widetilde{D}\in\mathcal{D}}\: V(G,\widetilde{D}) - \frac{\lambda}{2}\bigl\Vert \widetilde{D} - D^* \bigr\Vert^2.
\end{align}
We claim that $V(G^*,D^*) = V^{\operatorname{prox}}_{\lambda}(G^*,D^*)$. To show this claim, consider the definition of the $\lambda$-proximal equilibrium:
\begin{equation}
V^{\operatorname{prox}}_{\lambda}(G^*,D^*) := \max_{\tilde{D}\in\mathcal{D}}\, V(G^*,\tilde{D}) - \frac{\lambda}{2}\big\Vert D^* - \tilde{D} \big\Vert.
\end{equation}
Here $D^*$ maximizes the objective because we have assumed that $V(G^*,D) \le V(G^*,D^*)$ holds for every $D\in\mathcal{D}$. Therefore, the claim is valid and $V(G^*,D^*) = V^{\operatorname{prox}}_{\lambda}(G^*,D^*)$. 

Also, note that for every $D$ the solution $\tilde{D}$ in the proximal optimization satisfies $V^{\operatorname{prox}}_{\lambda}(G^*,D)\le V(G^*,\tilde{D})$. Combining these results with \eqref{prop2:proof the reverse direction}, we obtain the following inequalities which hold for every $G\in\mathcal{G}$ and $D\in\mathcal{D}$:
\begin{equation}
V^{\operatorname{prox}}_{\lambda}(G^*,D)\, \le \, V^{\operatorname{prox}}_{\lambda}(G^*,D^*) \, \le \, V^{\operatorname{prox}}_{\lambda}(G,D^*).     
\end{equation}
The above equation shows that the pair $(G^*,D^*)$ is a $\lambda$-proximal equilibrium.

\subsection{Proof of Proposition 3}
Consider a $\lambda_2$-proximal equilibrium $(G^*,D^*)$. As a result of Proposition 2, for every $G\in\mathcal{G}$ and $D\in\mathcal{D}$ we have
\begin{align}
        V(G^*,D)\, \le \, V(G^*,D^*)  
         \le \, \max_{\widetilde{D}\in\mathcal{D}}\: V(G,\widetilde{D}) - \frac{\lambda_2}{2}\bigl\Vert \widetilde{D} - D^* \bigr\Vert^2. \nonumber
    \end{align}
Since $\lambda_1\le \lambda_2$, the following holds
\begin{align}
         \max_{\widetilde{D}\in\mathcal{D}}\: V(G,\widetilde{D}) - \frac{\lambda_2}{2}\bigl\Vert \widetilde{D} - D^* \bigr\Vert^2 \,
         \le \, \max_{\widetilde{D}\in\mathcal{D}}\: V(G,\widetilde{D}) - \frac{\lambda_1}{2}\bigl\Vert \widetilde{D} - D^* \bigr\Vert^2, \nonumber
    \end{align}
which shows that
\begin{align}
        V(G^*,D)\, \le \, V(G^*,D^*) \, \le \, \max_{\widetilde{D}}\: V(G,\widetilde{D}) - \frac{\lambda_1}{2}\bigl\Vert \widetilde{D} - D^* \bigr\Vert^2. \nonumber
    \end{align}
Due to Proposition 2, $(G^*,D^*)$ will be a $\lambda_1$-proximal equilibrium as well. Hence, the proof is complete and we have
\begin{equation}
 \operatorname{PE}_{\lambda_2}(V) \subseteq \operatorname{PE}_{\lambda_1}(V).
\end{equation}

\subsection{Proof of Proposition 4}

Consider the definition of a $\lambda$-proximal equilibrium in the parameterized space:
\begin{equation}
    V_{\lambda}^{\operatorname{prox}}(G_{\boldsymbol{\theta}},D_{\mathbf{w}}) := \max_{\widetilde{\mathbf{w}}}\, V(G_{\boldsymbol{\theta}},D_{\widetilde{\mathbf{w}}}) - \frac{\lambda}{2}\bigl\Vert D_{\widetilde{\mathbf{w}}} - D_{\mathbf{w}} \bigr\Vert^2.
\end{equation}
In the above optimization problem, the first term $V(G_{\boldsymbol{\theta}},D_{\widetilde{\mathbf{w}}})$ is assumed to be $\eta_2$-smooth in $\widetilde{\mathbf{w}}$, while the second term $\frac{\lambda}{2}\Vert D_{\widetilde{\mathbf{w}}} - D_{\mathbf{w}} \Vert^2$ will be $\frac{\lambda}{2}\eta_1$-strongly convex in $\widetilde{\mathbf{w}}$. As a result, the sum of the two terms will be $(\frac{\lambda\eta_1}{2}-\eta_2)$-strongly concave if $\eta_2<\frac{\lambda\eta_1}{2}$ holds. Since the objective function is strongly-concave in $\widetilde{\mathbf{w}}$, it will be maximized by a unique solution $\mathbf{w}^*$. Moreover, applying the Danskin's theorem \cite{bertsekas1997nonlinear} implies that the following holds at the optimal $\mathbf{w}^*$:
\begin{equation}
    \nabla_{\theta} V_{\lambda}^{\operatorname{prox}}(G_{\boldsymbol{\theta}},D_{\mathbf{w}})  = \nabla_{\theta} V(G_{\boldsymbol{\theta}},D_{\mathbf{w}^*}).
\end{equation}

\subsection{Proof of Theorem 2}
\begin{lemma}\label{Lemma: Strong convex minimization}
Suppose that $f$ is a $\gamma$-strongly convex function according to norm $\Vert\cdot \Vert$, i.e. for any $x,y\in\operatorname{dom}(f)$ and $\lambda\in[0,1]$ we have
\begin{align}\label{Strong convexity lemma}
    f(\lambda x + (1-\lambda)y)\, \le\, \lambda f(x) +(1-\lambda)f(y) - \frac{\gamma\lambda(1-\lambda)}{2}\Vert x-y \Vert^2.
\end{align}
Consider the following optimization problem where the feasible set $\mathcal{X}$ is a convex set and $x^*$ is the optimal solution,
\begin{equation}
    \min_{x\in\mathcal{X}}\, f(x).
\end{equation}
Then, for every $x\in\mathcal{X}$ we have
\begin{equation}
    f(x) - f(x^*) \ge \frac{\gamma}{2}\Vert x- x^* \Vert^2.
\end{equation}
\end{lemma}
\begin{proof}
If we apply the strong-convexity definition \eqref{Strong convexity lemma} to $x\in\mathcal{X},\, x^*$ we obtain
\begin{align}
 f(\lambda x + (1-\lambda)x^*)\, \le\, \lambda f(x) +(1-\lambda)f(x^*) - \frac{\gamma\lambda(1-\lambda)}{2}\Vert x-x^* \Vert^2.   
\end{align}
The above inequality results in
\begin{align}
    f(\lambda x + (1-\lambda)x^*) - f(x^*)\, \le\, \lambda (f(x) -f(x^*))  - \frac{\gamma\lambda(1-\lambda)}{2}\Vert x-x^* \Vert^2.
\end{align}
Note that $\mathcal{X}$ is assumed to be a convex set and therefore $\lambda x + (1-\lambda)x^*\in\mathcal{X}$ implying $f(x^*)\le f(\lambda x + (1-\lambda)x^*)$. As a result, for every $0\le\lambda\le 1$
\begin{equation}
    \frac{\gamma\lambda(1-\lambda)}{2}\Vert x-x^* \Vert^2\, \le\, \lambda (f(x) -f(x^*)).
\end{equation}
Thus for every $0<\lambda\le 1$ we have
\begin{equation}
   \frac{\gamma(1-\lambda)}{2}\Vert x-x^* \Vert^2\, \le\, f(x) -f(x^*),
\end{equation}
which proves that $\frac{\gamma}{2}\Vert x-x^* \Vert^2 \le f(x) -f(x^*)$ and completes Lemma \ref{Lemma: Strong convex minimization}'s proof.
 \end{proof}
Based on Proposition 2 and the definition of $D^{\boldsymbol{\theta}}$, we only need to show that for the W2GAN's objective, which we denote by $V(G,D)$, the following holds for every $\boldsymbol{\theta}$:
\begin{equation}
 V(G_{\boldsymbol{\theta}^*},D^{\boldsymbol{\theta}^*})  \le \max_{D\in \mathcal{D}}\: V(G_{\boldsymbol{\theta}},D) - \frac{1}{2\eta}\bigl\Vert D - D^{\boldsymbol{\theta}^*} \bigr\Vert^2_{\dot{H}^1}.
\end{equation}
To show the above inequality, it suffices to prove the following inequality
\begin{equation}
 V(G_{\boldsymbol{\theta}^*},D^{\boldsymbol{\theta}^*})  \le  V(G_{\boldsymbol{\theta}},D^{\boldsymbol{\theta}}) - \frac{1}{2\eta}\bigl\Vert D^{\boldsymbol{\theta}} - D^{\boldsymbol{\theta}^*} \bigr\Vert^2_{\dot{H}^1}.
\end{equation}
\textbf{Claim:} For the W2GAN problem, we have $$V(G_{\boldsymbol{\theta}},D^{\boldsymbol{\theta}}) = \frac{1}{2\eta}\mathbb{E}\bigl[ \Vert \nabla D^{\boldsymbol{\theta}}(\mathbf{X})\Vert^2_2\bigr].$$

To show this claim, note that according to the W2GAN's formulation we have $V(G_{\boldsymbol{\theta}},D^{\boldsymbol{\theta}}) := W_c(P_{\mathbf{X}},P_{G_{\boldsymbol{\theta}}(\mathbf{Z})})$ where $c(\mathbf{x},\mathbf{x}')=\frac{\eta}{2}\Vert \mathbf{x}- \mathbf{x}' \Vert_2^2$ is the second-order cost function specified in the theorem. We start by proving this result for $\eta =1$. In this case, the Brenier theorem \cite{ambrosio2013user} proves that the optimal transport map from the data variable $\mathbf{X}$ to the generative model $G_{\boldsymbol{\theta}}(\mathbf{Z})$ can be derived from the gradient of the optimal $D^{\boldsymbol{\theta}}$ as follows
\begin{equation}
    \psi^{\operatorname{opt}}(\mathbf{x}) = \mathbf{x} - \nabla D^{\boldsymbol{\theta}} (\mathbf{x}).
\end{equation}
which plugged into the optimal transport objective $W_c(P,Q):= \inf_{\Pi(P,Q)}\mathbb{E}[c(\mathbf{X},\mathbf{X}')]$ proves that
$$ V(G_{\boldsymbol{\theta}},D^{\boldsymbol{\theta}}) := W_c(P_{\mathbf{X}},P_{G_{\boldsymbol{\theta}}(\mathbf{Z})}) = \mathbb{E}\bigl[ \frac{1}{2}\Vert \nabla D^{\boldsymbol{\theta}}(\mathbf{X})\Vert^2_2\bigr]. $$
The above equation proves the result holds for $\eta = 1$. For a general $\eta > 0$, note that applying a simple change of variable in the Kantorovich duality representation and solving the dual problem for $\tilde{D}(\mathbf{x})=\eta D(\mathbf{x})$ shows that $\psi^{\operatorname{opt}}(\mathbf{x}) = \mathbf{x} - \frac{1}{\eta}\nabla \tilde{D}^{\boldsymbol{\theta}}(\mathbf{x})$ transport samples from the data domain to the generative model. This is due to the fact that after applying this change of variable the Kantorovich duality reduces to $\eta$ multiplied to the dual problem for $\eta =1$. As a result, applying the transport map to the definition of the optimal transport cost shows that
\begin{equation*}
 V(G_{\boldsymbol{\theta}},D^{\boldsymbol{\theta}}) := W_c(P_{\mathbf{X}},P_{G_{\boldsymbol{\theta}}(\mathbf{Z})}) = \frac{1}{2\eta}\mathbb{E}\bigl[ \Vert \nabla D^{\boldsymbol{\theta}}(\mathbf{X})\Vert^2_2\bigr],   
\end{equation*}
proving the claim holds for any $\eta>0$.

Substituting the discriminator maximization with the result in the above claim, the W2GAN problem reduces to the following problem:
\begin{equation}
    \min_{\boldsymbol{\theta}}\: \frac{1}{2\eta}\mathbb{E}\bigl[ \big\Vert \nabla D^{\boldsymbol{\theta}}(\mathbf{X})\big\Vert^2_2\bigr].
\end{equation}
Here we can equivalently optimize for $D^{\boldsymbol{\theta}}\in \{ D^{\boldsymbol{\theta}}:\, \theta\in \Theta \}$ instead of minimizing over the variable $\boldsymbol{\theta}$, obtaining
\begin{equation}\label{Proof: W2GAN substituted obj}
    \min_{D^{\boldsymbol{\theta}}}\: \frac{1}{2\eta}\mathbb{E}\bigl[ \big\Vert \nabla D^{\boldsymbol{\theta}}(\mathbf{X})\big\Vert^2_2\bigr].
\end{equation}
Note that the term $\frac{1}{2}\mathbb{E}\bigl[ \big\Vert \nabla D^{\boldsymbol{\theta}}(\mathbf{X})\big\Vert^2_2\bigr] = \frac{1}{2}\big\Vert D^{\boldsymbol{\theta}} \big\Vert^2_{\dot{H}^1}$ reduces to the squared of the defined Sobolev norm in a semi-Hilbert space, which results in a $1$-strongly convex function according to $\Vert \cdot \Vert_{\dot{H}^1} $ with strong convexity defined as in \eqref{Strong convexity lemma}. As a result, the objective in \eqref{Proof: W2GAN substituted obj} is $\frac{1}{\eta}$-strongly convex according to the Sobolev norm $\Vert \cdot \Vert_{\dot{H}^1}$. In addition, this objective  is minimized over a convex feasible set $\{ D^{\boldsymbol{\theta}}:\, \theta\in \Theta \}$, due to the theorem's assumption. Therefore, Lemma \ref{Lemma: Strong convex minimization} shows that the optimal $D^{\boldsymbol{\theta}^*}$ satisfies the following inequality for every $\boldsymbol{\theta}$:
\begin{align*}
  \frac{1}{2\eta}\mathbb{E}\bigl[ \big\Vert \nabla D^{\boldsymbol{\theta}}(\mathbf{X})\big\Vert^2_2\bigr] -  \frac{1}{2\eta}\mathbb{E}\bigl[ \big\Vert \nabla D^{\boldsymbol{\theta}^*}(\mathbf{X})\big\Vert^2_2\bigr]\:
  \ge \:\frac{1}{2\eta} \big\Vert D^{\boldsymbol{\theta}} - D^{\boldsymbol{\theta}^*} \big\Vert^2_{\dot{H}^1}.
\end{align*}
The above result implies that
\begin{equation}
 V(G_{\boldsymbol{\theta}^*},D^{\boldsymbol{\theta}^*})  \le  V(G_{\boldsymbol{\theta}},D^{\boldsymbol{\theta}}) - \frac{1}{2\eta}\bigl\Vert D^{\boldsymbol{\theta}} - D^{\boldsymbol{\theta}^*} \bigr\Vert^2_{\dot{H}^1},
\end{equation}
which completes the proof.

\subsection{Proof of Theorem 3}
\begin{lemma}\label{Lemma: thm 4}
Consider two vectors $\mathbf{x}, \mathbf{y}$ with equal Euclidean norms $\Vert \mathbf{x}\Vert_2 = \Vert \mathbf{y}\Vert_2$. Then for every $0\le a \le b$, we have
\begin{equation}
    a\Vert\mathbf{x}- \mathbf{y} \Vert_2 \le \Vert a\mathbf{x}- b\mathbf{y}\Vert_2.
\end{equation}
\end{lemma}
\begin{proof}
Note that
\begin{align*}
  \Vert a\mathbf{x}- b\mathbf{y}\Vert^2_2  - a^2\Vert \mathbf{x}- \mathbf{y}\Vert^2_2 \:
  =\: & (b^2-a^2)\Vert\mathbf{y}\Vert^2_2 - 2a(b-a)\mathbf{x}^T\mathbf{y} \\
  =\: & (b-a)( (b+a)\Vert\mathbf{y}\Vert^2_2 - 2a\mathbf{x}^T\mathbf{y}) \\
  \ge\; &0.
\end{align*}
The above holds as we have assumed that $0 \le a\le b$ implying $0\le 2a\le b+a$ and since the two vectors $\mathbf{x},\mathbf{y}$ share the same Euclidean norm
\begin{equation*}
    |\mathbf{x}^T\mathbf{y}|\le \frac{\Vert\mathbf{x} \Vert_2^2+\Vert\mathbf{y} \Vert_2^2}{2} =\Vert\mathbf{y} \Vert_2^2.
\end{equation*}
Hence, Lemma \ref{Lemma: thm 4}'s proof is complete.
\end{proof}
As shown by the Kantorovich duality \cite{villani2008optimal}, for the optimal $D^{\boldsymbol{\theta}}$ and the optimal coupling $\pi_{\boldsymbol{\theta}}(P_{\mathbf{X}},P_{G_{\boldsymbol{\theta}}(\mathbf{Z})})$ the following holds with probability $1$ for every joint sample $(\mathbf{X}, \mathbf{X}')$ drawn from the optimal coupling $\pi_{\boldsymbol{\theta}}$,
\begin{equation}
    D^{\boldsymbol{\theta}}(\mathbf{X}) - D^{\boldsymbol{\theta}}(\mathbf{X}')  = \Vert \mathbf{X} - \mathbf{X}' \Vert_2.   
\end{equation}
Knowing that $D^{\boldsymbol{\theta}}$ is 1-Lipschitz, for every convex combination $\beta \mathbf{X} + (1-\beta)\mathbf{X}'$ we must have $$\nabla D^{\boldsymbol{\theta}}\bigl(\beta \mathbf{X} + (1-\beta)\mathbf{X}'\bigr) = \frac{\mathbf{X} - \mathbf{X}'}{\Vert \mathbf{X} - \mathbf{X}' \Vert}. $$ This will imply that there definitely exists $\alpha_{\boldsymbol{\theta}}$ such that the transport map described in the theorem maps the data distribution to the generative model. Plugging this transport map into the definition of the first-order Wasserstein distance, we obtain
\begin{align*}
    V(G_{\boldsymbol{\theta}},D^{\boldsymbol{\theta}}):&= W_1(P_{\mathbf{X}},P_{G_{\boldsymbol{\theta}}(\mathbf{Z})}) \\
    &= \mathbb{E}\bigl[ \big\Vert {\alpha}^2 _{\boldsymbol{\theta}}(\mathbf{X})\nabla D^{\boldsymbol{\theta}}(\mathbf{X}) \big\Vert_2\bigr] \\
    &= \mathbb{E}\bigl[ \big\Vert \alpha_{\boldsymbol{\theta}}(\mathbf{X}) \nabla D^{\boldsymbol{\theta}}(\mathbf{X}) \big\Vert_2^2\bigr]
\end{align*}
where the last equality holds since the Euclidean norm of $\nabla D^{\boldsymbol{\theta}}(\mathbf{X})$ has a unit Euclidean norm with probability $1$ over the data distribution $P_{\mathbf{X}}$ as we proved $\nabla D^{\boldsymbol{\theta}}(\beta \mathbf{X} + (1-\beta)\mathbf{X}') = \frac{\mathbf{X} - \mathbf{X}'}{\Vert \mathbf{X} - \mathbf{X}' \Vert} $ holds for every $0\le \beta\le 1$ including $\beta=1$.

As a result, the Wasserstein GAN problem reduces to the following optimization problem
\begin{equation}\label{proof WGAN thm 3}
    \min_{\boldsymbol{\theta}}\: \mathbb{E}\bigl[ \big\Vert \alpha_{\boldsymbol{\theta}}(\mathbf{X}) \nabla D^{\boldsymbol{\theta}}(\mathbf{X}) \big\Vert_2^2\bigr]
\end{equation}
Defining $h_{\boldsymbol{\theta}}(\mathbf{X}) := \alpha_{\boldsymbol{\theta}}(\mathbf{X}) \nabla D^{\boldsymbol{\theta}}(\mathbf{X}) $, $\frac{1}{2}\mathbb{E}\bigl[\Vert h_{\boldsymbol{\theta}}(\mathbf{X}) \Vert_2^2\bigr]$ is $1$-strongly convex with respect to the norm function $$\Vert h \Vert_{\dot{H}^0}=\sqrt{\mathbb{E}\bigl[\, h^2(\mathbf{X})\,\bigr]}$$ that is induced by the following inner product and results in a Hilbert space
$$\langle D_1,D_2 \rangle_{\dot{H}^0} := \mathbb{E}_{P_{\mathbf{X}}}[D_1(\mathbf{X}) D_2(\mathbf{X})].$$
Therefore, for the $\boldsymbol{\theta}^*$ minimizing the objective in \eqref{proof WGAN thm 3} over the assumed convex set $\{h_{\boldsymbol{\theta}}:\, \boldsymbol{\theta}\in\Theta \}$, Lemma \ref{Lemma: Strong convex minimization} implies that
\begin{align*}
    \mathbb{E}\bigl[ \big\Vert \alpha_{\boldsymbol{\theta}}(\mathbf{X}) \nabla D^{\boldsymbol{\theta}}(\mathbf{X}) \big\Vert_2^2\bigr] - \mathbb{E}\bigl[ \big\Vert \alpha_{\boldsymbol{\theta}^*}(\mathbf{X}) \nabla D^{\boldsymbol{\theta}^*}(\mathbf{X}) \big\Vert_2^2\bigr]\:
    =\: &\mathbb{E}\bigl[ \big\Vert h_{\boldsymbol{\theta}}(\mathbf{X}) \big\Vert_2^2\bigr] - \mathbb{E}\bigl[ \big\Vert h_{\boldsymbol{\theta}^*}(\mathbf{X}) \big\Vert_2^2\bigr]  \\
    \ge\: & \Vert h_{\boldsymbol{\theta}} - h_{\boldsymbol{\theta}^*} \Vert_{\dot{H}^0}^2 \\
    =\: & \mathbb{E}\bigl[ \Vert h_{\boldsymbol{\theta}}(\mathbf{X}) - h_{\boldsymbol{\theta}^*}(\mathbf{X}) \Vert_2^2\bigr] \\
    =\: & \mathbb{E}\bigl[ \Vert \alpha_{\boldsymbol{\theta}}(\mathbf{X}) \nabla D^{\boldsymbol{\theta}}(\mathbf{X}) - \alpha_{\boldsymbol{\theta}^*}(\mathbf{X}) \nabla D^{\boldsymbol{\theta}^*}(\mathbf{X}) \Vert_2^2\bigr] \\
    \ge\: & \frac{\eta}{2} \mathbb{E}\bigl[ \Vert  \nabla D^{\boldsymbol{\theta}}(\mathbf{X}) -  \nabla D^{\boldsymbol{\theta}^*}(\mathbf{X}) \Vert_2^2\bigr] \\
    =\: & \frac{\eta}{2} \Vert D_{\boldsymbol{\theta}}- D_{\boldsymbol{\theta}^*} \Vert^2_{\dot{H}^1}.
\end{align*}
Here the last inequality follows from Lemma \ref{Lemma: thm 4} since every $D^{\boldsymbol{\theta}}$ has a unit-norm gradient with probability $1$ according to the data distribution $P_{\mathbf{X}}$. Therefore, we have proved that
\begin{equation}
    V(G_{\boldsymbol{\theta}},D^{\boldsymbol{\theta}}) - V(G_{\boldsymbol{\theta}^*},D^{\boldsymbol{\theta}^*}) \ge \frac{\eta}{2} \Vert D_{\boldsymbol{\theta}}- D_{\boldsymbol{\theta}^*} \Vert^2_{\dot{H}^1}. 
\end{equation}
The above inequality results in the following for every feasible $\boldsymbol{\theta}$
\begin{equation}
  V(G_{\boldsymbol{\theta}^*},D^{\boldsymbol{\theta}^*}) \le \max_{D\in\mathcal{D}} V(G_{\boldsymbol{\theta}},D) - \frac{\eta}{2} \Vert D - D_{\boldsymbol{\theta}^*} \Vert^2_{\dot{H}^1}.
\end{equation}
Hence, according to Proposition 2, we have shown that the pair $(G_{\boldsymbol{\theta}^*},D^{\boldsymbol{\theta}^*})$ is an $\eta$-proximal equilibrium with respect to the Sobolev norm $\Vert\cdot \Vert_{\dot{H}^1}$. 

\end{appendices}

\end{document}